\documentclass{article}
\pdfoutput=1
\usepackage{arxiv}
\setcitestyle{authoryear,open={(},close={)}}
\renewcommand{\cite}{\citep}

\usepackage[utf8]{inputenc} %
\usepackage[T1]{fontenc}    %
\usepackage{hyperref}
\usepackage{url}
\usepackage{booktabs}       % professional-quality tables
\usepackage{amsfonts}       % blackboard math symbols
\usepackage{nicefrac}       % compact symbols for 1/2, etc.
\usepackage{microtype}      % microtypography
\usepackage{dsfont}
\usepackage{graphicx}
\usepackage{amsmath}
\usepackage{mathtools}
\usepackage{amssymb}
\usepackage{amsthm}
\usepackage{cases}
\usepackage{color}
\usepackage{textcomp}
\usepackage{xcolor}
\usepackage{caption}
\usepackage{bbm}
\usepackage{comment}
\usepackage{xspace}
\usepackage{multirow}
\usepackage{subcaption}
\usepackage{authblk}
\usepackage{xspace}
\usepackage{algorithm}
\usepackage{algpseudocode}
\usepackage{adjustbox}
\usepackage{wrapfig}

\def\shownotes{0}
\ifnum\shownotes=1
\newcommand{\authnote}[2]{[#1: #2]}
\else
\newcommand{\authnote}[2]{}
\fi

\newcommand{\pref}{p_\text{ref}}
\newcommand{\pproxy}{p_\theta}

\newcommand{\examplexi}{X_z[i]}
\newcommand{\unif}{\textup{unif}}

\newcommand{\pstar}{p^*}

\newcommand{\lambdaxprimex}{\lambda_z(x)}
\newcommand{\hatthetaz}{\hat{\theta}_z}
\newcommand{\thetaxprimex}{\hatthetaz(x)}

\newcommand{\weight}{domain weight\xspace}

\newcommand{\weights}{domain weights\xspace}
\newcommand{\Weights}{Domain weights\xspace}
\newcommand{\algname}{DoReMi\xspace}

\newcommand{\indicator}{\mathbf{1}}

\newlength{\widebarargwidth}
\newlength{\widebarargheight}
\newlength{\widebarargdepth}

\newcommand\R{\ensuremath{\mathbb{R}}} % Real numbers
\ifthenelse{\isundefined{\definition}}{\newtheorem{definition}{Definition}}{}
\ifthenelse{\isundefined{\assumption}}{\newtheorem{assumption}{Assumption}}{}
\ifthenelse{\isundefined{\hypothesis}}{\newtheorem{hypothesis}{Hypothesis}}{}
\ifthenelse{\isundefined{\proposition}}{\newtheorem{proposition}{Proposition}}{}
\ifthenelse{\isundefined{\theorem}}{\newtheorem{theorem}{Theorem}}{}
\ifthenelse{\isundefined{\lemma}}{\newtheorem{lemma}{Lemma}}{}
\ifthenelse{\isundefined{\corollary}}{\newtheorem{corollary}{Corollary}}{}
\ifthenelse{\isundefined{\alg}}{\newtheorem{alg}{Algorithm}}{}
\ifthenelse{\isundefined{\example}}{\newtheorem{example}{Example}}{}
\newcommand{\E}{\ensuremath{\mathbb{E}}} % Expectation
\begin{document}
\bibliographystyle{plainnat}

\title{DoReMi: Optimizing Data Mixtures Speeds Up Language Model Pretraining}
\author[1,2]{Sang Michael Xie\thanks{Work done while interning at Google DeepMind. Correspondence: <Sang Michael Xie: xie@cs.stanford.edu>, <Hieu Pham: hyhieu@gmail.com>, <Adams Wei Yu: adamsyuwei@google.com>.}}
\author[1]{Hieu Pham}
\author[1]{Xuanyi Dong}
\author[1]{Nan Du}
\author[1]{Hanxiao Liu}
\author[1]{Yifeng Lu}
\author[2]{\\Percy Liang}
\author[1]{Quoc V. Le}
\author[2]{Tengyu Ma}
\author[1]{Adams Wei Yu}

\affil[1]{Google DeepMind}
\affil[2]{Stanford University}

\date{}

\newcommand{\fix}{\marginpar{FIX}}
\newcommand{\new}{\marginpar{NEW}}

\maketitle

\begin{abstract}
The mixture proportions of pretraining data domains (e.g., Wikipedia, books, web text) greatly affect language model (LM) performance. In this paper, we propose Domain Reweighting with Minimax Optimization (DoReMi), which first trains a small proxy model using group distributionally robust optimization (Group DRO) over domains to produce domain weights (mixture proportions) without knowledge of downstream tasks. We then resample a dataset with these domain weights and train a larger, full-sized model. In our experiments, we use DoReMi on a 280M-parameter proxy model to set the domain weights for training an 8B-parameter model (30x larger) more efficiently. On The Pile, DoReMi improves perplexity across \textit{all} domains, even when it downweights a domain. DoReMi improves average few-shot downstream accuracy by 6.5\% points over a baseline model trained using The Pile's default domain weights and reaches the baseline accuracy with 2.6x fewer training steps. On the GLaM dataset, DoReMi, which has no knowledge of downstream tasks, even matches the performance of using domain weights tuned on downstream tasks.
\end{abstract}

\section{Introduction}
\label{sec:intro}
Datasets for training language models (LMs) are typically sampled from a mixture of many domains~\citep{gao2020pile,du2021glam,chowdhery2022palm,brown2020gpt3}.
For example, The Pile~\citep{gao2020pile}, a large publicly available dataset, is composed of 24\% web data, 9\% Wikipedia, 4\% GitHub, etc.\footnote{The \weights, which are based on token count in this paper, varies by tokenizer; see Appendix~\ref{app:training-details}.}
The composition of the pretraining data greatly affects the effectiveness of an LM~\citep{du2021glam,hoffmann2022chinchilla,xie2023data}.
However, it is unclear how much of each domain to include to produce a model that performs well for a wide variety of downstream tasks.

Existing works determine \weights (the sampling probabilities for each domain) by using intuition or a set of downstream tasks.
For example, The Pile uses heuristically-chosen \weights, which could be suboptimal.
On the other hand, existing LMs such as PaLM~\citep{chowdhery2022palm} and GLaM~\citep{du2021glam} tune the \weights based on a set of downstream tasks, but requires training potentially thousands of LMs on different \weights and risks overfitting to the particular set of downstream tasks.

\begin{figure}
\centering
\includegraphics[width=\textwidth]{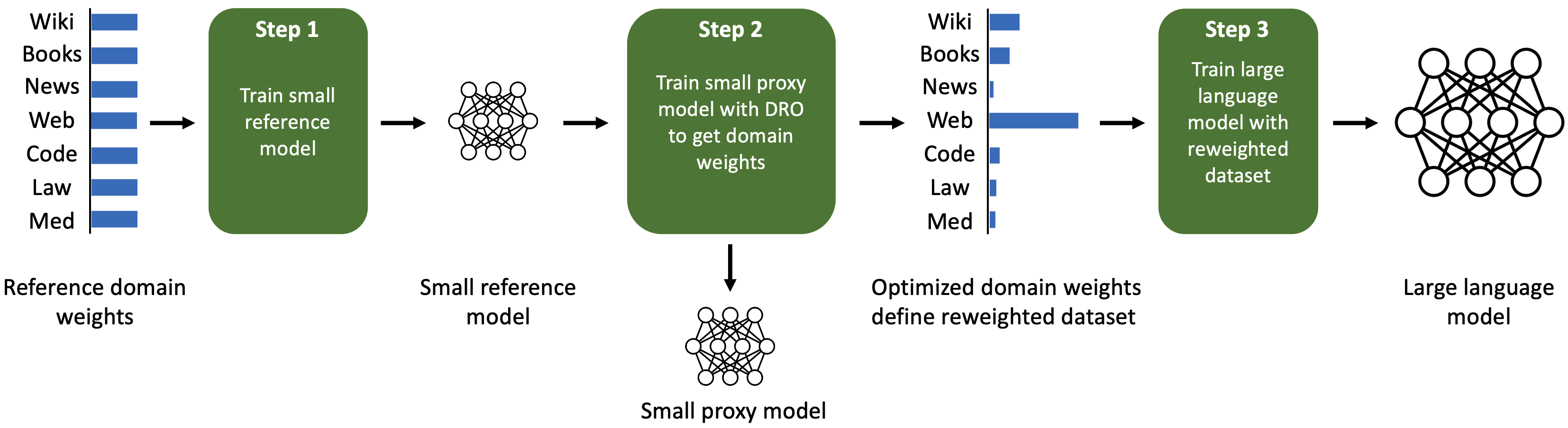}
\caption{Given a dataset with a set of domains,
Domain Reweighting with Minimax Optimization (\algname)
optimizes the \weights to improve language models trained on the dataset.
First, \algname uses some initial reference \weights to train a reference model (Step 1).
The reference model is used to guide the training of a small proxy model using group distributionally robust optimization (Group DRO) over domains~\citep{oren2019drolm,sagawa2020group,nemirovski2009robust}, which we adapt to output \weights instead of a robust model (Step 2).
We then use the tuned \weights to train a large model (Step 3).
}
\label{fig:fig1}
\end{figure}

Instead of optimizing \weights based on a set of downstream tasks, our approach aims to find \weights which lead to models that perform well on all domains by minimizing the worst-case \textit{excess loss} over domains, following~\citet{oren2019drolm,mindermann2022prioritized}.
The excess loss is the loss gap between the model being evaluated and a pretrained reference model.

This motivates our algorithm, \textbf{Do}main \textbf{Re}weighting with \textbf{Mi}nimax Optimization (\algname), which leverages distributionally robust optimization (DRO) to tune the \weights without knowledge of downstream tasks (Figure~\ref{fig:fig1}).
First, \algname trains a small reference model (e.g., 280M parameters) in a standard way.
Second, \algname trains a small distributionally robust language model (DRO-LM) \citep{oren2019drolm}, which minimizes the worst-case excess loss (relative to the reference's model's loss) across all domains.
Notably, \textit{rather than using the robust LM}, we take the \weights produced by DRO training.
Finally, we train a large (8B) LM on a new dataset defined by these \weights.

Our approach adapts the DRO-LM framework~\citep{oren2019drolm} to optimize \weights instead of producing a robust model.
To do this, \algname uses the online learning-based optimizer from Group DRO~\citep{sagawa2020group,nemirovski2009robust}, which dynamically updates \weights according to the loss on each domain for rescaling the training objective, instead of sub-selecting examples from a minibatch as in \citet{oren2019drolm,mindermann2022prioritized}. Finally, \algname takes the averaged \weights over DRO training steps.

\begin{figure}
\centering
\hspace{1.5cm}
\includegraphics[width=0.75\textwidth]{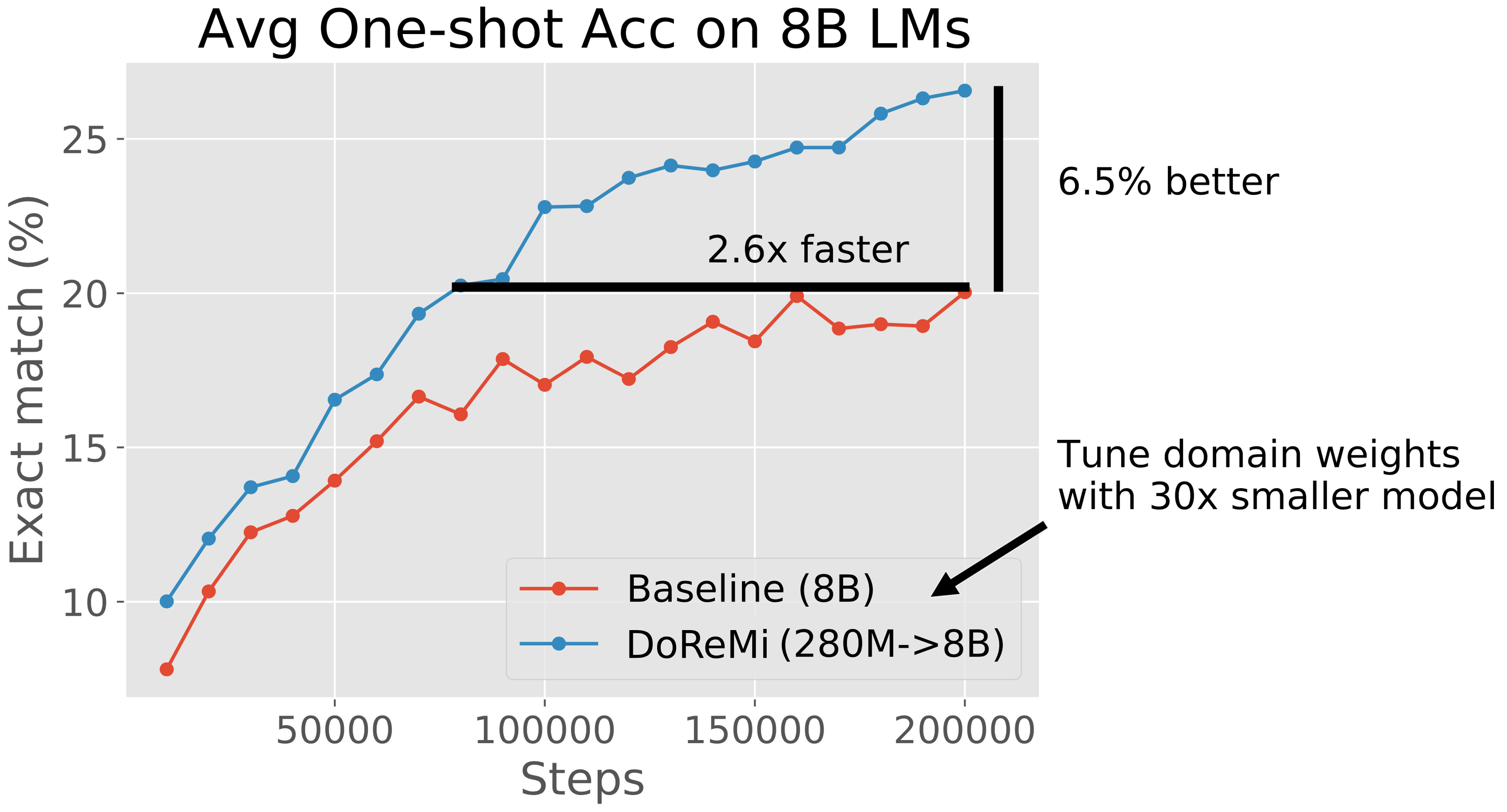}
\caption{\algname optimizes \weights with a small model (280M params) and uses these \weights to train a much larger model (8B params, 30x larger). Here, optimizing the \weights (training a small model twice) takes 8\% of the compute of training the large model.
\algname improves average one-shot downstream accuracy by 6.5\% points and reaches the baseline accuracy 2.6x faster when pretraining on The Pile. }
\label{fig:annotated_improvement}
\end{figure}

In Section~\ref{sec:experiments}, we run \algname on 280M proxy and reference models to optimize \weights on The Pile~\citep{gao2020pile} and the GLaM dataset~\citep{du2021glam} (used in PaLM~\citep{chowdhery2022palm}).
The \algname \weights are used to train an 8B parameter LM (over 30x larger).
On The Pile, \algname reduces perplexity on \textit{all} domains over baseline \weights, even when it downweights a domain.
\algname improves average downstream accuracy over a baseline model trained on The Pile's default \weights by 6.5\% points on generative few-shot tasks and achieves the baseline downstream accuracy 2.6x faster (Figure~\ref{fig:annotated_improvement}). 
In Section~\ref{sec:analysis}, we find that \algname consistently improves LM training when varying the sizes of the proxy model and the main model trained with optimized \weights.
On the GLaM dataset where \weights tuned on downstream tasks are available, \algname even performs comparably to tuning \weights on downstream task performance.\footnote{A public re-implementation of \algname and optimized \weights for The Pile can be found at \url{https://github.com/sangmichaelxie/doremi}.} 
\section{Domain Reweighting with Minimax Optimization (\algname)}
\label{sec:method}
In this section we define \algname, an algorithm for using a small proxy model to optimize the \weights of a language modeling dataset, which then improves the training of a large model.

\paragraph{Setup.}
Suppose that we have $k$ domains (e.g., Wikipedia, GitHub), where for each domain $i$, we have a set of examples $D_i$.
\Weights $\alpha\in \Delta^k$ specify a probability distribution over the $k$ domains, and consequently a distribution over the training data: $P_\alpha = \sum_{i=1}^k \alpha_i \cdot \unif(D_i)$
where $\unif(D) = \frac{1}{|D|} \sum_{x\in D}\delta_x$ is the uniform distribution over the examples in $D$ and $\delta_x(x')$ is 1 if $x'=x$ and 0 otherwise.

\paragraph{\algname.}
The inputs of \algname are the data $D_1,\dots,D_k$, reference \weights $\alpha_{\text{ref}}$ (e.g., uniform or based on raw token count of each domain), and training hyperparameters for the large, full-size model (number of training steps $T$ and batch size $b$). \algname returns optimized \weights $\bar{\alpha}$ and ultimately, a large model trained on $P_{\bar{\alpha}}$.

\paragraph{Step 1: Obtain a small reference model.}
We first train a model $\pref$ on some reference \weights $\alpha_{\text{ref}}$ (e.g., based on raw token count as a default) for $T$ steps, batch size $b$.
This model serves as the reference model for step 2 and captures a baseline level of difficulty of each example/domain.
The reference model can be a relatively small model (280M parameters in our experiments).

\paragraph{Step 2: Train proxy model with Group DRO to obtain \weights.}
To obtain \weights, we train a small \textit{proxy model} $\pproxy$ in the distributionally robust language modeling (DRO-LM)~\citep{oren2019drolm} framework with the Group DRO optimizer~\citep{sagawa2020group}, where $\theta$ are the weights of the proxy model.
This framework trains a robust model by optimizing the worst-case loss over domains, which is equivalent to the following minimax objective:
\begin{align}
\min_\theta \max_{\alpha \in \Delta^k} L(\theta, \alpha) \coloneqq \sum_{i=1}^k \alpha_i \cdot \left[ \frac{1}{\sum_{x\in D_i} |x| }\sum_{x\in D_i} \ell_\theta(x) - \ell_{\text{ref}}(x)\right]
\end{align}
where the losses $\ell_\theta(x) = -\log~\pproxy(x)$ and $\ell_{\text{ref}}(x) = -\log~\pref(x)$ are the negative log-likelihoods of the proxy and reference models respectively in this paper, and $|x|$ is the number of tokens in an example $x$.
The objective aims to minimize the worst-case excess loss across domains because the inner maximization over $\alpha$ puts all the weight on the domain with the highest excess loss.

Intuitively, the excess loss ($\ell_\theta(x) - \ell_{\text{ref}}(x)$) measures the headroom for the proxy model to improve, with respect to the reference model, on example $x$.
Examples with higher excess loss are those where the reference model achieves low loss (such that the example is ``learnable'') but the proxy model still has high loss.
Examples with low excess loss may be very high entropy (i.e. optimal loss is high, and thus the reference loss is high) or very low entropy (i.e., easy to learn, and thus the proxy loss is low).
The Group DRO optimizer works by interleaving exponentiated gradient ascent updates on \weights $\alpha_t$ with gradient updates on the proxy model weights $\theta_t$ over training steps $t$.
The optimizer updates $\alpha_t$ to upweight domains with high excess loss, which scales up the proxy model's gradient update on examples from these domains.
Following \citet{nemirovski2009robust}, we return the average weights over the training trajectory $\bar{\alpha} = \frac{1}{T}\sum_{i=1}^T \alpha_t$ as the optimized \weights to use in step 3.

\paragraph{Step 3: Train large model with new \weights.}
The tuned \weights $\bar{\alpha}$ define a new training distribution $P_{\bar{\alpha}}$. We resample the data from this new distribution to train a main model (larger than the reference/proxy models), using a standard training procedure.

\begin{algorithm}
\caption{\algname domain reweighting (Step 2)}
\label{alg:alg1}
\begin{algorithmic}
\Require Domain data $D_1,\dots, D_k$, number of training steps $T$, batch size $b$, step size $\eta$, smoothing parameter $c \in [0,1]$ (e.g., $c=$1e-3 in our implementation).
\State Initialize proxy weights $\theta_0$
\State Initialize \weights $\alpha_0 = \frac{1}{k}\indicator$
\For{$t$ from 1 to $T$}
    \State Sample minibatch $B=\{x_1,\dots,x_j\}$ of size $b$ from $P_u$, where $u=\frac{1}{k}\indicator$
    \State Let $|x|$ be the token length of example $x$ ($|x|\leq L$)
    \State Compute per-domain excess losses for each domain $i\in \{1,2, \dots, k\}$ ($\ell_{\theta, j}(x)$ is $j$-th token-level loss):\\
        $~~~~~~~~~~~~~\lambda_t[i] \gets \frac{1}{\sum_{x \in B \cap D_i} |x|} \sum_{x \in B \cap D_i}\sum_{j=1}^{|x|}  \max\{\ell_{\theta_{t-1}, j}(x) - \ell_{\text{ref}, j}(x), 0\}$
    \State Update \weights ($\exp$ is entrywise): $\alpha'_t \gets \alpha_{t-1}\exp(\eta \lambda_t)$
    \State Renormalize and smooth \weights: $\alpha_t \gets (1-c)\frac{\alpha'_t}{\sum_{i=1}^k \alpha'_t[i]}  + cu$
    \State Update proxy model weights $\theta_t$ for the objective $L(\theta_{t-1}, \alpha_t)$ (using Adam, Adafactor, etc.)
\EndFor\\
\Return $\frac{1}{T}\sum_{t=1}^T \alpha_t$
\end{algorithmic}
\end{algorithm}

\paragraph{Details for Step 2.}
Algorithm~\ref{alg:alg1} provides the pseudocode for Step 2.
The main structure of Algorithm~\ref{alg:alg1} is a training loop which updates the proxy model over $T$ steps.
At each step, we follow~\citet{sagawa2020group} and sample a minibatch with uniform \weights (regardless of the reference \weights $\alpha_{\text{ref}}$, which only affects the reference model).
We then compute the per-domain excess losses, normalized by the total number of tokens in each domain, and use them to update the \weights $\alpha_t$ at each step.
We first compute the per-domain excess loss at a per-token level and then aggregate, where the token-level losses at index $j$ are $\ell_{\theta_{t-1},j}(x)=-\log p_{\theta_{t-1}}(x_j \mid x_1, \dots, x_{j-1})$ and $\ell_{\text{ref}, j}(x) = -\log~ \pref(x_j \mid x_1,\dots, x_{j-1})$.
Since the Group DRO optimizer~\citep{sagawa2020group} requires a non-negative loss, we clip the per-token excess loss at 0.
Finally, we update the proxy model for the objective $L(\theta_{t-1}, \alpha_t)$ using a standard optimizer such as Adam~\citep{kingma2015adam} or Adafactor~\citep{shazeer2018adafactor}. All experiments in this paper use Adafactor.
We set the \weight update step size to $\eta=1$ and the smoothing parameter to $c=$1e-3 in all our experiments and did not extensively tune these hyperparameters.

\paragraph{Iterated \algname.}
We extend \algname by running it for multiple rounds, setting the reference \weights $\alpha_{\text{ref}}$ for the next round to be $\bar{\alpha}$ from the previous round.
We call this \textit{iterated \algname}.
The entire iterated process still only uses small models for tuning \weights.
We stop iterating when the \weights converge, which we define as when maximum change in any \weight $\|\bar{\alpha} - \alpha_{\text{ref}}\|_\infty$ is less than 1e-3. 
Empirically, this takes only 3 rounds on the GLaM dataset (Section~\ref{sec:main-results-8B}).
\section{\algname Improves LM Training Efficiency and Performance}

\label{sec:experiments}
In this section, we use \algname \weights optimized with a 280M-parameter proxy model to train a 8B-parameter main model (30x larger).
We consider two datasets, The Pile~\citep{gao2020pile} and the GLaM dataset~\citep{du2021glam}.
On The Pile, \algname reduces perplexity significantly on every domain, improves average downstream accuracy on generative one-shot tasks by 6.5\%, and achieves the baseline accuracy 2.6x faster.
On the GLaM dataset where \weights tuned on downstream datasets are available, \algname finds \weights with comparable performance to downstream-tuned \weights.

\subsection{Experimental setup}
\paragraph{The Pile dataset.}
The Pile~\citep{gao2020pile} is a 800GB text dataset with 22 domains (Table~\ref{tab:pile-mixture-weights}).
The default \weights were determined heuristically.
We use the default \weights from The Pile dataset to train the baseline and as the reference \weights $\alpha_{\text{ref}}$ in \algname (see Appendix~\ref{app:training-details}).

\paragraph{GLaM dataset.}
The GLaM dataset~\citep{du2021glam} (also used in training PaLM~\citep{chowdhery2022palm}) includes text from 8 domains (Table~\ref{tab:glam-mixture-weights}).
For comparison, the GLaM \weights (downstream-tuned) were tuned according to the downstream performance of models trained on each domain and the size of each domain~\citep{du2021glam}.
We consider this an oracle comparison, since these \weights are tuned on downstream tasks that are in our evaluation set.
We use uniform \weights both for training the baseline and the reference \weights $\alpha_{\text{ref}}$ for \algname.

\paragraph{Training setup.}
We train Transformer~\citep{vaswani2017attention} decoder-only LMs with the standard next-token language modeling loss.
We conduct a controlled comparison by equalizing the amount of compute, measured by the number of tokens processed during training.
For The Pile, we train each model for 200k steps; for the GLaM dataset, we train each model for 300k steps.
All models use a batch size of 512 and maximum token length of 1024.
The proxy and reference models have 280M parameters.
All models are trained from scratch (other hyperparameters are in Appendix~\ref{app:training-details}).

\paragraph{Evaluation.}
We use held-out validation data to measure the perplexity on each domain.
For downstream evaluation, we use the generative one-shot tasks from the GPT-3 paper~\citep{brown2020gpt3}: TriviaQA~\citep{joshi2017triviaqa}, NaturalQuestions~\citep{kwiatkowski2019natural}, WebQuestions~\citep{berant2013freebase}, SQuADv2~\citep{rajpurkar2018squadrun}, and LAMBADA~\citep{paperno2016lambada}.
We use the standard exact-match accuracy metric for the these datasets.
The performance on these datasets (particularly TriviaQA) has been shown to correlate well with model scale even at the 100M--1B range~\citep{brown2020gpt3}.

\paragraph{Compute used for optimizing \weights.}
We train two 280M models (the reference and proxy models) to optimize the \weights.
This is 8\% of the FLOPs required to train the main 8B model. All FLOPs come from standard forward and backward passes.

\paragraph{Notation for model sizes in \algname.}
We denote the size of the reference/proxy models (which are always the same size in our experiments) and the size of the main model trained with \algname \weights as ``\algname (size of reference/proxy$\rightarrow$size of main model)'': for example, \algname (280M$\rightarrow$8B). When we are discussing the optimized \weights independently of the main model, we only include one number (e.g., \algname (280M)) which refers to the reference/proxy model size.
\begin{figure}[b]
\centering
\begin{subfigure}{0.49\textwidth}
\centering
\includegraphics[width=\textwidth]{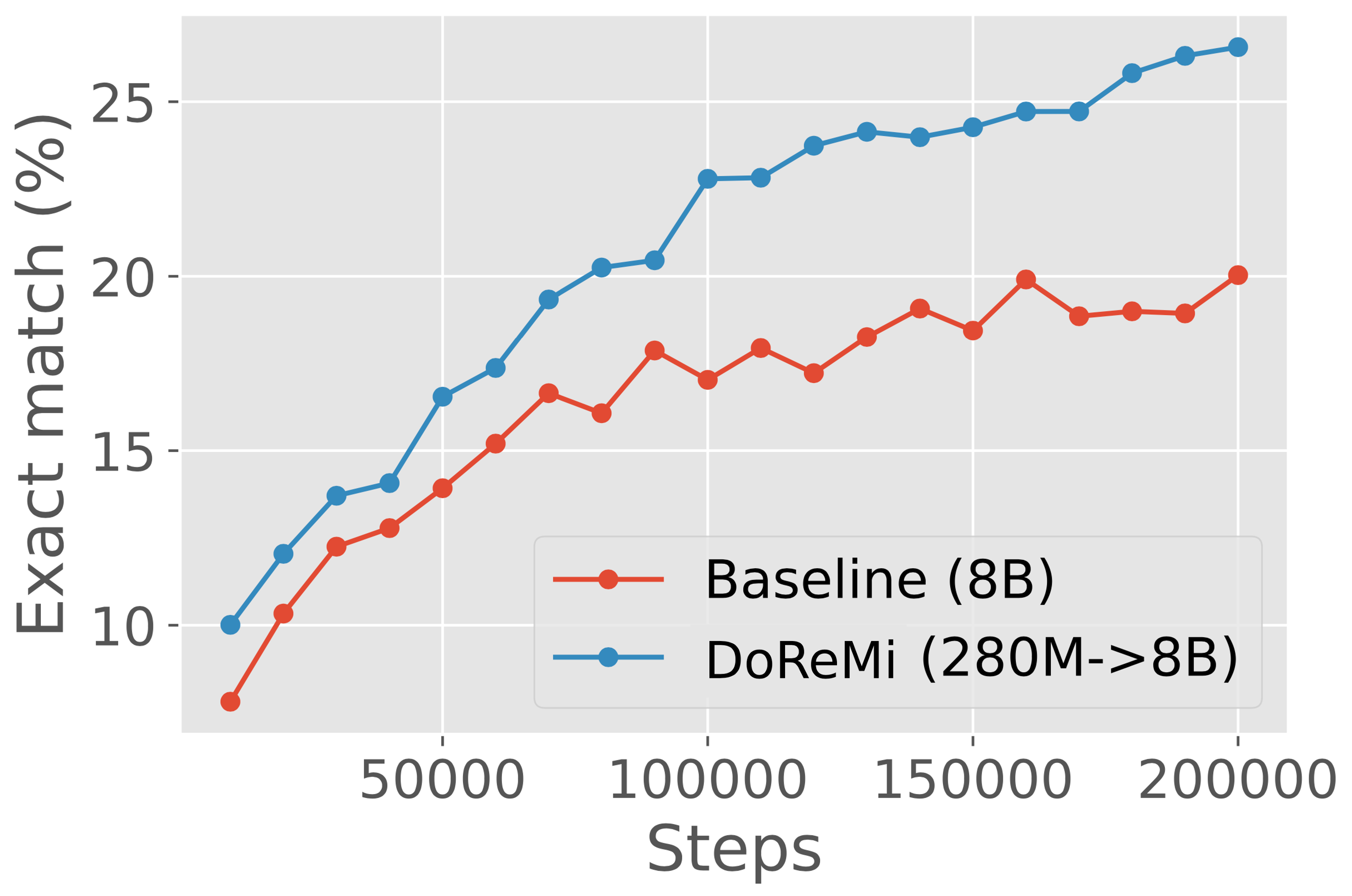}
\caption{The Pile}
\end{subfigure}
\hfill
\begin{subfigure}{0.49\textwidth}
\centering
\includegraphics[width=\textwidth]{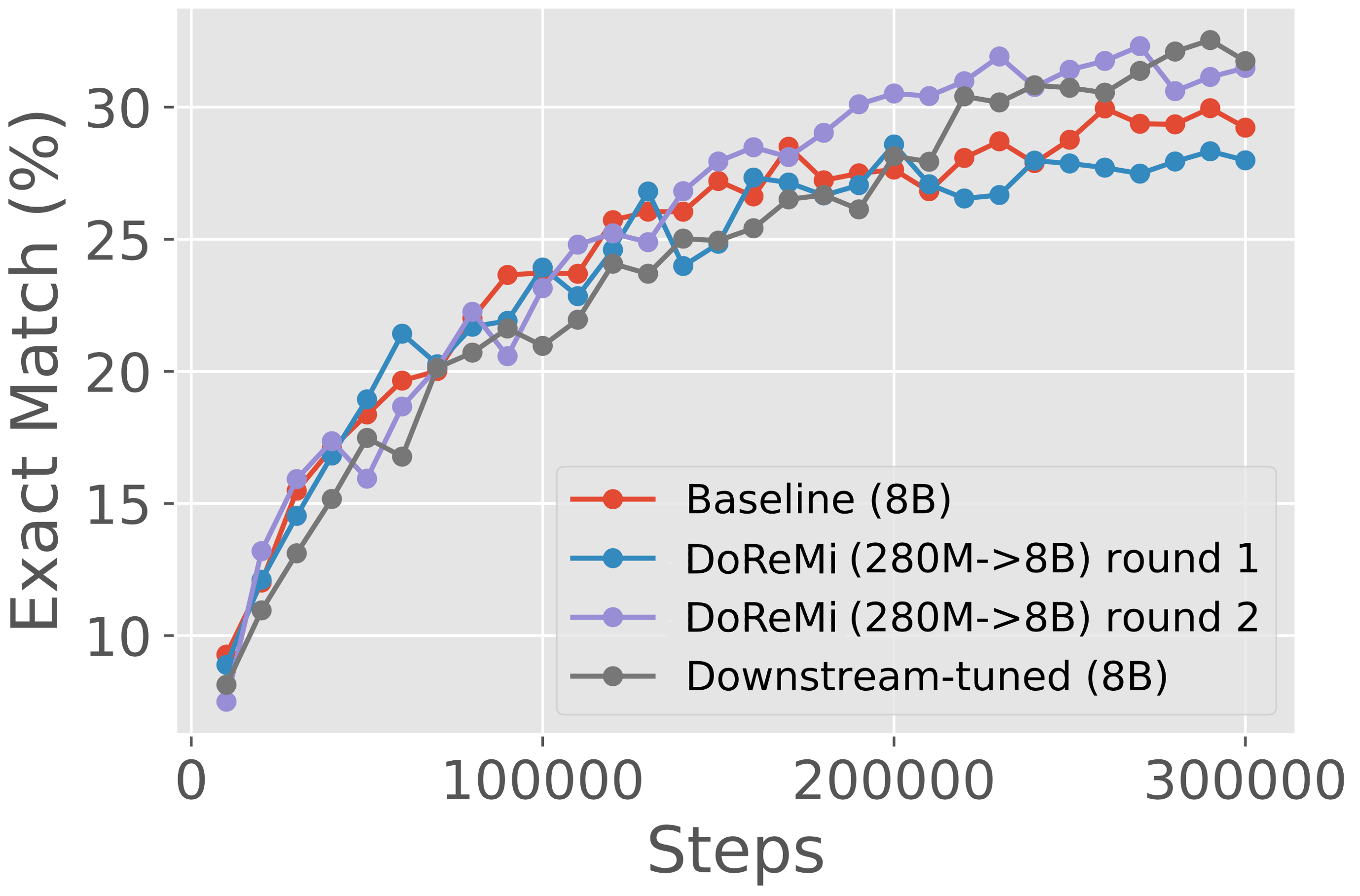}
\caption{GLaM dataset}
\end{subfigure}
\caption{Average one-shot downstream accuracy (exact match) on 5 tasks, with 8B parameter models trained on The Pile (left) and the GLaM dataset (right). On The Pile, \algname improves downstream accuracy by 6.5\% points and achieves the baseline accuracy 2.6x faster (same plot as Figure~\ref{fig:annotated_improvement}).
On the GLaM dataset, iterated \algname (round 2) attains comparable performance to oracle \weights tuned with downstream tasks that are in our evaluation set.
}
\label{fig:280m-8b-fewshot}
\end{figure}

\begin{figure}
\centering
\includegraphics[width=0.88\textwidth]{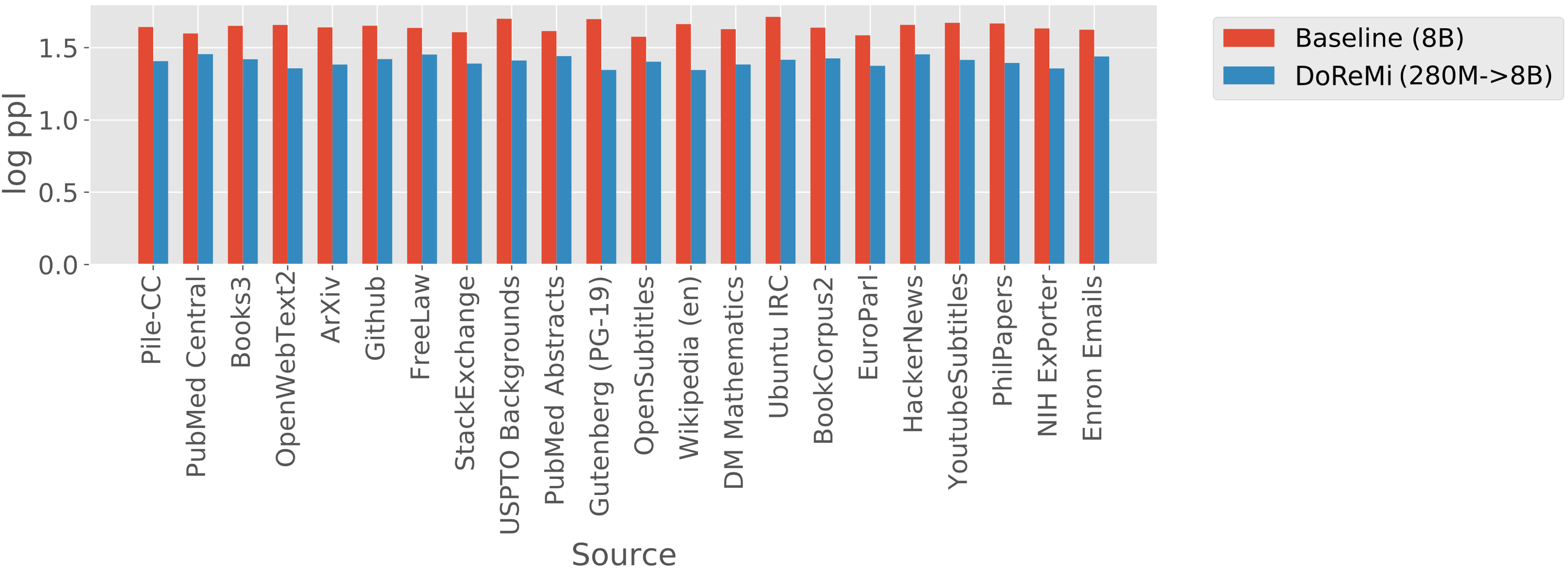}
\caption{Per-domain log-perplexity of 8B models on The Pile. Despite downweighting some domains, \algname improves log-perplexity on all domains.}
\label{fig:280m-8b-pile-perplexity}
\end{figure}

\subsection{\algname improves perplexity and downstream accuracy}
\label{sec:main-results-8B}
We show that \algname significantly improves both the perplexity and downstream accuracy of 8B models trained on The Pile and the GLaM dataset over their respective baseline \weights.

\paragraph{Downstream accuracy improves on The Pile.}
Figure~\ref{fig:280m-8b-fewshot} (left) shows the average downstream performance for baseline and \algname (280M$\rightarrow$8B) models on The Pile. \algname improves the downstream accuracy by 6.5\% points and achieves the baseline accuracy within 75k steps --- 2.6x faster than the baseline (200k steps).
Thus, \algname can dramatically speed up training and improve downstream performance.

\paragraph{\algname can reduce perplexity across all domains without a tradeoff.}
Figure~\ref{fig:280m-8b-pile-perplexity} shows the per-domain log-perplexity of the 8B models on The Pile.
\algname significantly reduces the perplexity over the baseline across \textit{all} domains, despite allocating lower weight to some domains.
How can this occur?
One hypothesis is that the domains with the lowest and highest entropy can be downweighted without impacting the perplexity much.
The lowest entropy domains statistically require few samples to learn.
The highest entropy domains have token distributions that are close to common uniform priors --- for example, models at random initialization tend to output a uniform next token distribution. Thus, we need less samples to fit these domains. Positive transfer from allocating more samples to medium entropy domains can then improve perplexity on all domains.
In Appendix~\ref{sec:simple-example}, we provide a simple example where reweighting domains can improve perplexity on all domains and \algname finds such \weights in simulations.

\begin{table}[t]
\caption{\Weights on The Pile. Baseline \weights are computed from the default Pile dataset. \algname (280M) uses a 280M proxy model to optimize the \weights.
}
\label{tab:pile-mixture-weights}
\centering
\begin{adjustbox}{max width=0.432\textwidth}
\begin{tabular}{lrrrr}
\toprule
                Domain  & Baseline & \algname (280M) & Difference \\
                  \midrule
Pile-CC           & 0.1121            & 0.6057                 & {\color[HTML]{38761D} +0.4936}           \\
YoutubeSubtitles  & 0.0042            & 0.0502                 & {\color[HTML]{38761D} +0.0460}           \\
PhilPapers        & 0.0027            & 0.0274                 & {\color[HTML]{38761D} +0.0247}           \\
HackerNews        & 0.0075            & 0.0134                 & {\color[HTML]{38761D} +0.0059}           \\
Enron Emails      & 0.0030            & 0.0070                 & {\color[HTML]{38761D} +0.0040}           \\
EuroParl          & 0.0043            & 0.0062                 & {\color[HTML]{38761D} +0.0019}           \\
Ubuntu IRC        & 0.0074            & 0.0093                 & {\color[HTML]{38761D} +0.0019}           \\
BookCorpus2       & 0.0044            & 0.0061                 & {\color[HTML]{38761D} +0.0017}           \\
NIH ExPorter      & 0.0052            & 0.0063                 & {\color[HTML]{38761D} +0.0011}           \\
OpenSubtitles     & 0.0124            & 0.0047                 & {\color[HTML]{CC0000} -0.0077}          \\
Gutenberg (PG-19) & 0.0199            & 0.0072                 & {\color[HTML]{CC0000} -0.0127}          \\
\bottomrule
\end{tabular}
\end{adjustbox}
\begin{adjustbox}{max width=0.448\textwidth}
\begin{tabular}{lrrrr}
\toprule
                Domain  & Baseline & \algname (280M) & Difference\\
                  \midrule
DM Mathematics    & 0.0198            & 0.0018                 & {\color[HTML]{CC0000} -0.0180}          \\
Wikipedia (en)    & 0.0919            & 0.0699                 & {\color[HTML]{CC0000} -0.0220}          \\
OpenWebText2      & 0.1247            & 0.1019                 & {\color[HTML]{CC0000} -0.0228}          \\
Github            & 0.0427            & 0.0179                 & {\color[HTML]{CC0000} -0.0248}          \\
FreeLaw           & 0.0386            & 0.0043                 & {\color[HTML]{CC0000} -0.0343}          \\
USPTO Backgrounds & 0.0420            & 0.0036                 & {\color[HTML]{CC0000} -0.0384}          \\
Books3            & 0.0676            & 0.0224                 & {\color[HTML]{CC0000} -0.0452}          \\
PubMed Abstracts  & 0.0845            & 0.0113                 & {\color[HTML]{CC0000} -0.0732}          \\
StackExchange     & 0.0929            & 0.0153                 & {\color[HTML]{CC0000} -0.0776}          \\
ArXiv             & 0.1052            & 0.0036                 & {\color[HTML]{CC0000} -0.1016}          \\
PubMed Central    & 0.1071            & 0.0046                 & {\color[HTML]{CC0000} -0.1025}                     \\
\bottomrule
\end{tabular}
\end{adjustbox}
\end{table}

\begin{table}[t]
\vspace{-0.1in}
\caption{\Weights in the GLaM dataset. Iterated \algname (280M) converges within 3 rounds, with a similar overall pattern to \weights tuned on downstream tasks.
}
\label{tab:glam-mixture-weights}
\centering
\begin{adjustbox}{max width=0.6\textwidth}
\begin{tabular}{lrrrr}
\toprule
                  & Round 1 & Round 2 & Round 3 & Downstream-tuned \\ \midrule
Wikipedia         & 0.09                      & 0.05                      & 0.05                      & 0.06             \\
Filtered webpages & 0.44                      & 0.51                      & 0.51                      & 0.42             \\
Conversations     & 0.10                      & 0.22                      & 0.22                      & 0.27             \\
Forums            & 0.16                      & 0.04                      & 0.04                      & 0.02             \\
Books             & 0.11                      & 0.17                      & 0.17                      & 0.20             \\
News              & 0.10                      & 0.02                      & 0.02                      & 0.02            \\
\bottomrule
\end{tabular}
\end{adjustbox}
\end{table}

\paragraph{Iterated \algname achieves performance of downstream-tuned weights on the GLaM dataset.}
We employ iterated \algname on the GLaM dataset over 3 rounds. We find that the second and third round \weights are almost identical (Table~\ref{tab:glam-mixture-weights}).
 Figure~\ref{fig:280m-8b-fewshot} (right) shows one-shot results for the first two rounds of iterated \algname.
After the first round, the \algname main model has comparable downstream accuracy to the baseline (uniform \weights).
After the second round, the \algname main model achieves comparable downstream accuracy to oracle \weights tuned on downstream tasks in our evaluation set.
Overall, domain reweighting has a smaller effect on GLaM, possibly because there are only 8 domains compared to 22 in The Pile.

\paragraph{Inspecting the \algname \weights.}
Tables~\ref{tab:pile-mixture-weights} and~\ref{tab:glam-mixture-weights} present the \algname \weights for The Pile and the GLaM dataset.
When running \algname on a 280M proxy model (\algname (280M)), most weight is put on the diverse Pile-CC web text domain. Note that Wikipedia is downweighted in comparison to the baseline, but \algname still improves the downstream accuracy on tasks derived from Wikipedia (e.g., TriviaQA, Appendix Table~\ref{tab:pile-8b-pertask-generative}). \Weights for a 1B proxy model (Appendix~\ref{tab:pile-mixture-weights-1B}) shows a different trend, where OpenWebText is the mostly upweighted instead of Pile-CC. This suggests that there may be multiple possible local minima in the \weight space.
On the GLaM dataset, the \algname weights have the same general pattern as the downstream-tuned \weights. \algname is able to recover a similar set of \weights by starting from uniform initial reference \weights, without any use of downstream data.

\section{Ablations and Analysis Across Scales}
\label{sec:analysis}
Previously in Section~\ref{sec:experiments}, we showed that \algname finds \weights using 280M models that can improve training of 8B models.
In this section, we conduct an analysis of \algname where we vary the scale of the proxy model in relation to the main model and ablate the components of the excess loss objective.

\begin{figure}
\centering
\begin{subfigure}{0.47\textwidth}
\centering
\includegraphics[width=\textwidth]{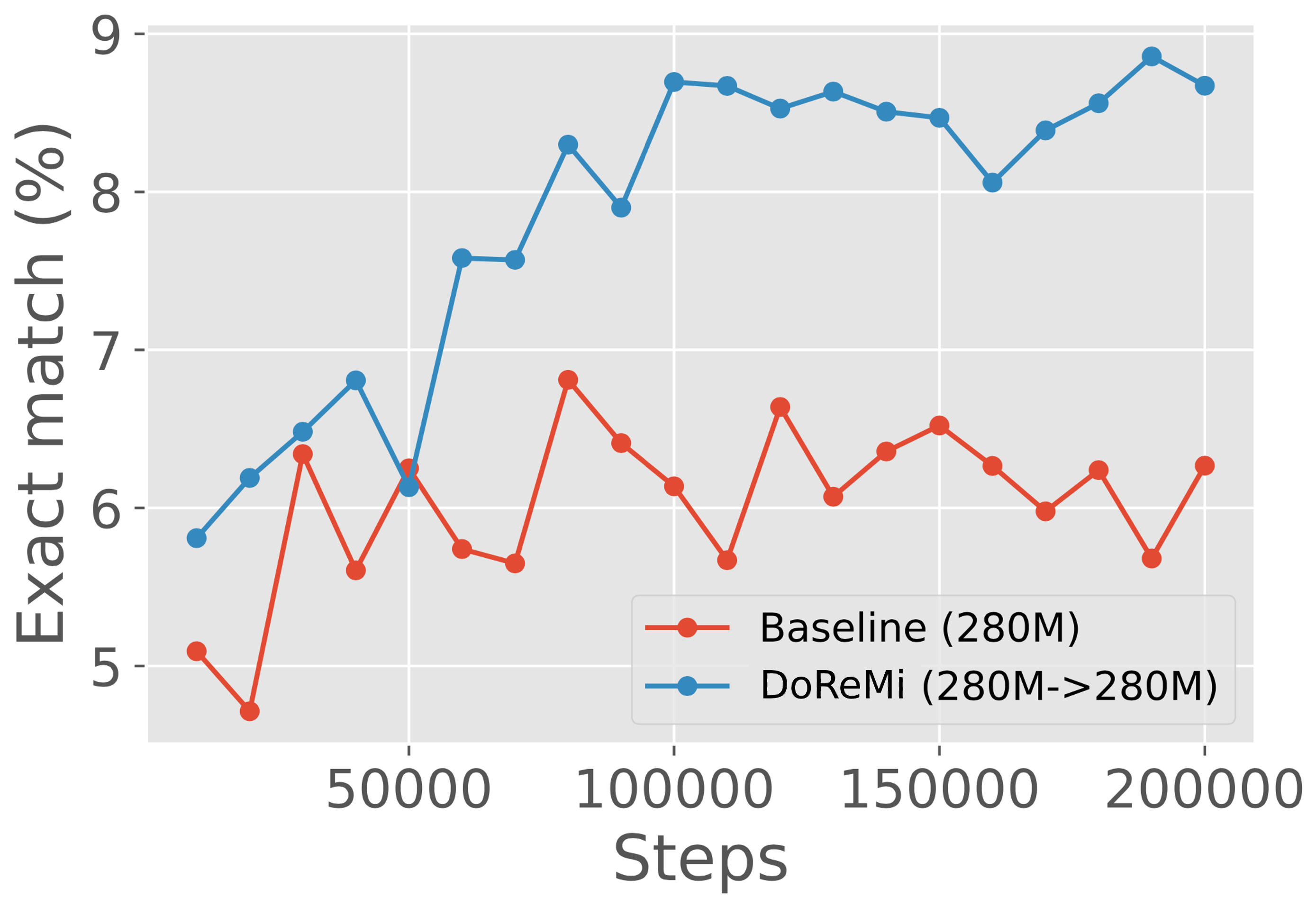}
\end{subfigure}
\hfill
\begin{subfigure}{0.47\textwidth}
\centering
\includegraphics[width=\textwidth]{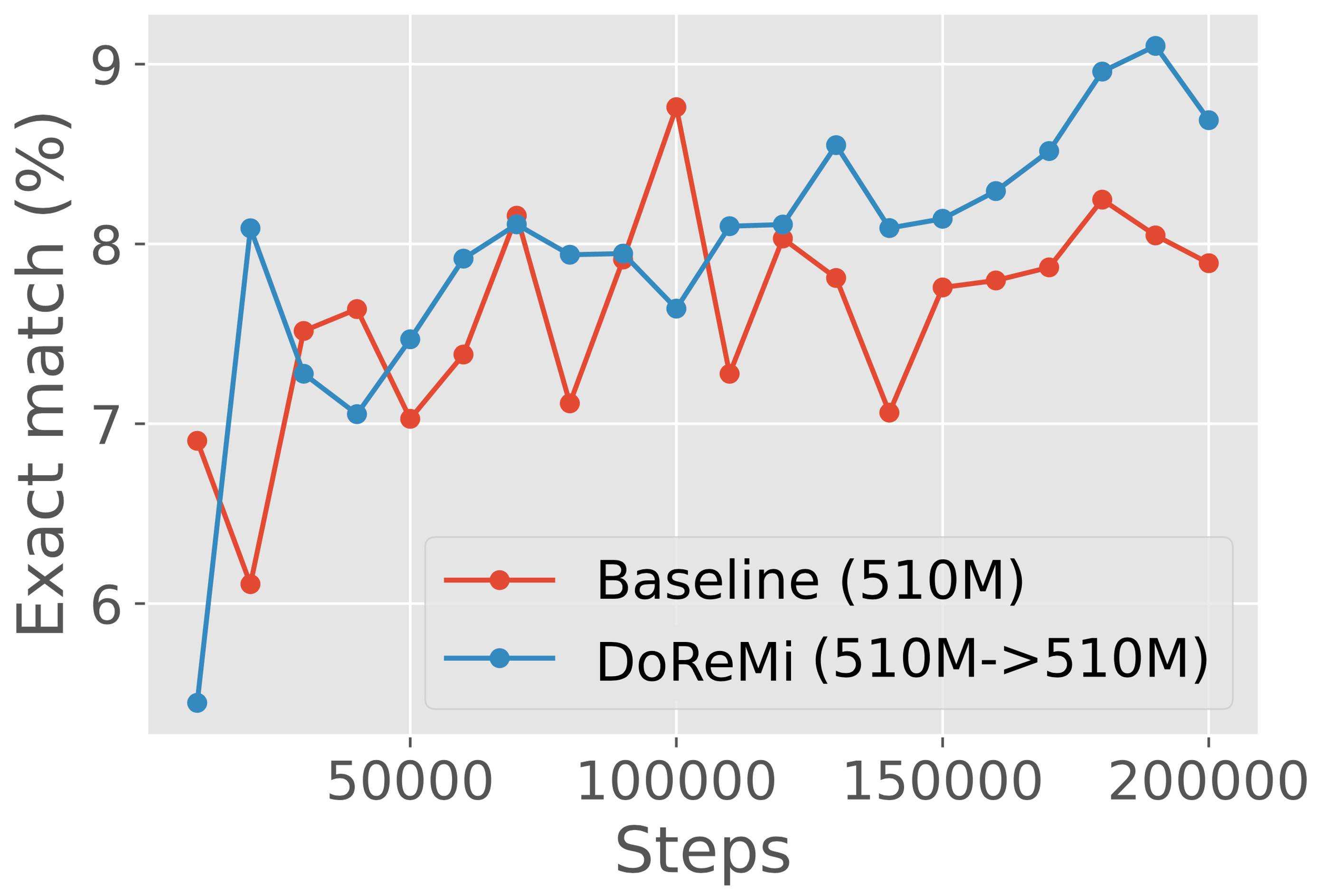}
\end{subfigure}
\hfill
\begin{subfigure}{0.47\textwidth}
\centering
\includegraphics[width=\textwidth]{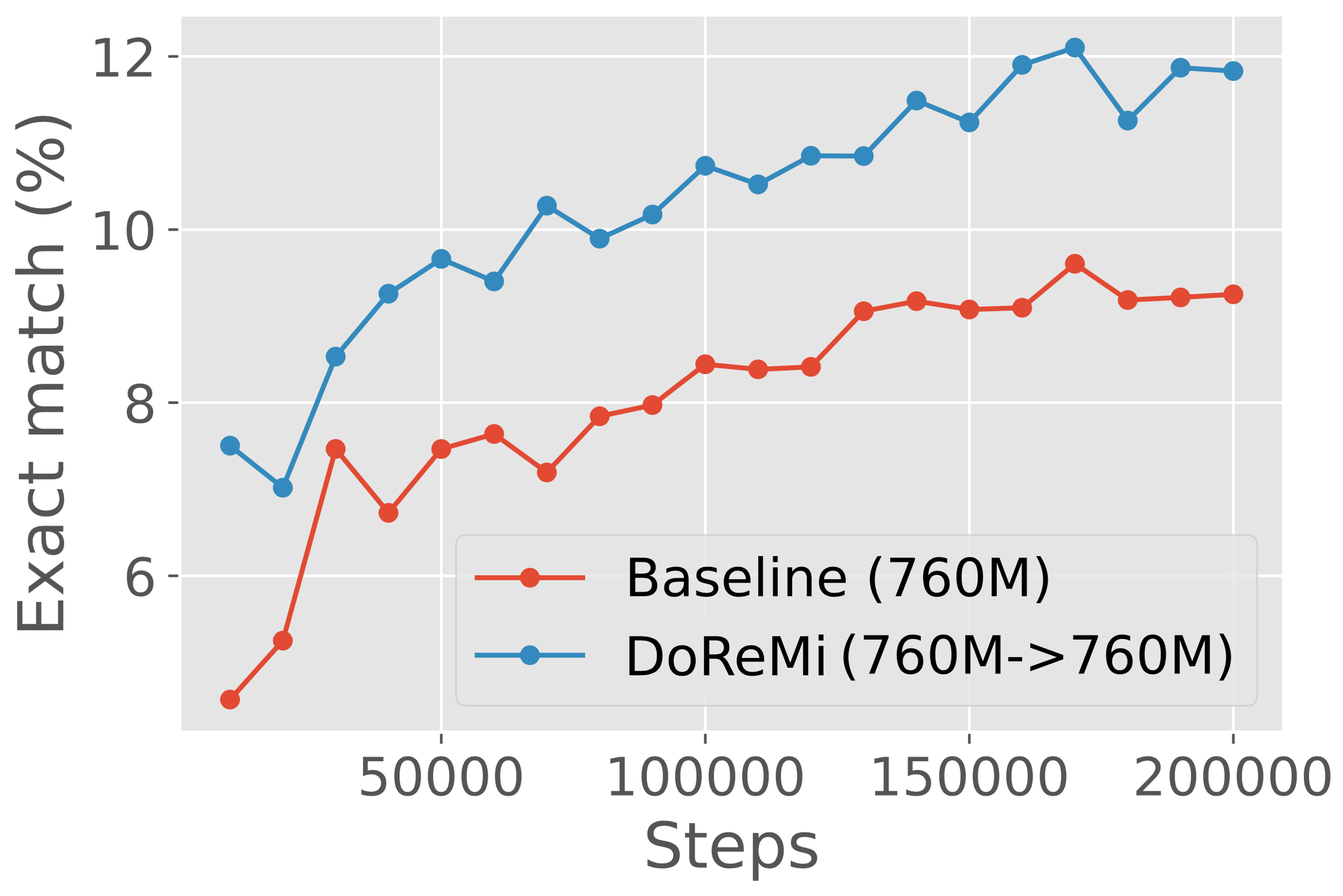}
\end{subfigure}
\hfill
\begin{subfigure}{0.47\textwidth}
\centering
\includegraphics[width=\textwidth]{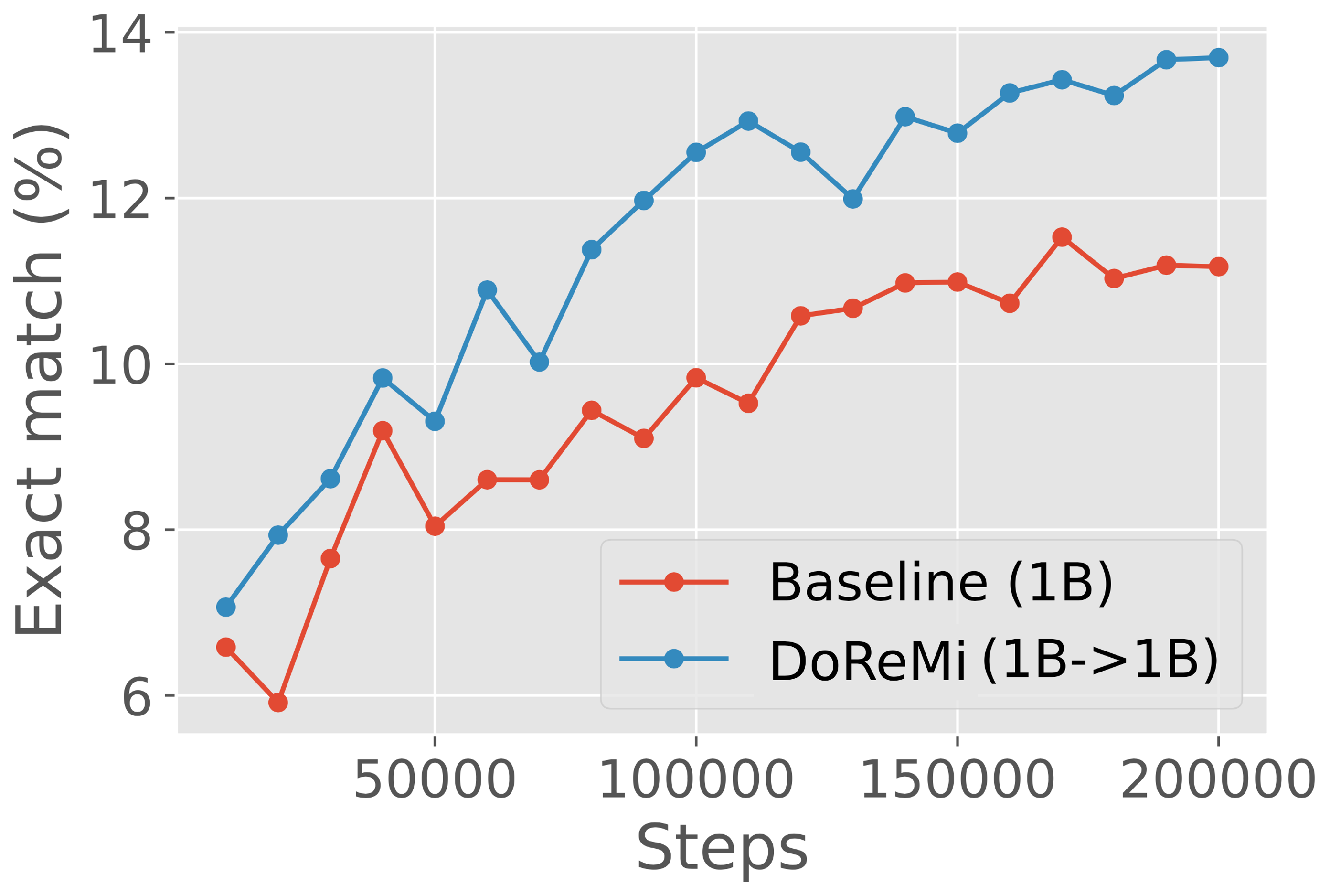}
\end{subfigure}
\caption{Average one-shot downstream accuracy across 4 model scales (280M, 510M, 760M, 1B) where the reference/proxy models for \algname are the same size as the main model trained with \algname \weights. \algname consistently improves downstream accuracy across scales, with a similar 3\% accuracy gap at 200k steps at most scales (except for 510M). \algname achieves the baseline accuracy 4x faster on average across scales.
}
\label{fig:pile-scaling}
\end{figure}

\paragraph{\algname improves LMs consistently across scales.}
We consider using proxy and main models of the same size to analyze \algname's behavior in a simple setting, without the need for the \weights to generalize across scales.
Note that this is just for scientific purposes since this does not save compute in practice.
In particular, we run \algname (X$\rightarrow$X) where X is 280M, 510M, 760M, or 1B on The Pile.
Figure~\ref{fig:pile-scaling} shows that \algname consistently improves downstream accuracy over the baseline by 2\% and achieves the baseline accuracy 4x faster on average across scales, and this improvement does not shrink with larger model size.
\algname improves the worst-case perplexity on all scales and improves 18 of 22 individual domain perplexities on average across scales (Appendix Table~\ref{tab:scaling-perplexity-pile}).
These experiments give a rough picture of how much is lost when using a smaller proxy model; our \algname(280M$\rightarrow$8B) model achieves the baseline accuracy 2.6x faster, while matching the proxy and main model sizes results in a 4x average speedup.

\paragraph{Proxy model underperforms main model, especially at larger sizes.}
Recall that \algname uses Group DRO to train a proxy model, which reweights the objective with the \weights.
In contrast, the main model is trained by resampling on the \weights from \algname. 
When the proxy model and the main model are the same size, which one is the better model?
Table~\ref{tab:pile-perplexity-drolm} shows that the proxy model typically underperforms the main model in this case.
The gap between the proxy and main model increases with scale, as the 1B proxy model not only underperforms the 1B main model but also the 1B baseline model, while the 280M proxy model achieves better perplexity than the 280M baseline model on 19/22 domains.
Despite the relatively poor quality of the 1B proxy model, the \weights still allow the 1B main model to achieve the baseline performance over 2x faster.
This suggests that \algname can succeed even if the proxy model is not trained well.
However, we hypothesize that the mismatch between the proxy and main model training (loss reweighting vs. resampling) explains their performance difference and therefore a resampling-based Group DRO optimizer may improve \algname for larger proxy models.

\begin{figure}
\centering
\begin{subfigure}{0.47\textwidth}
\centering
\includegraphics[width=\textwidth]{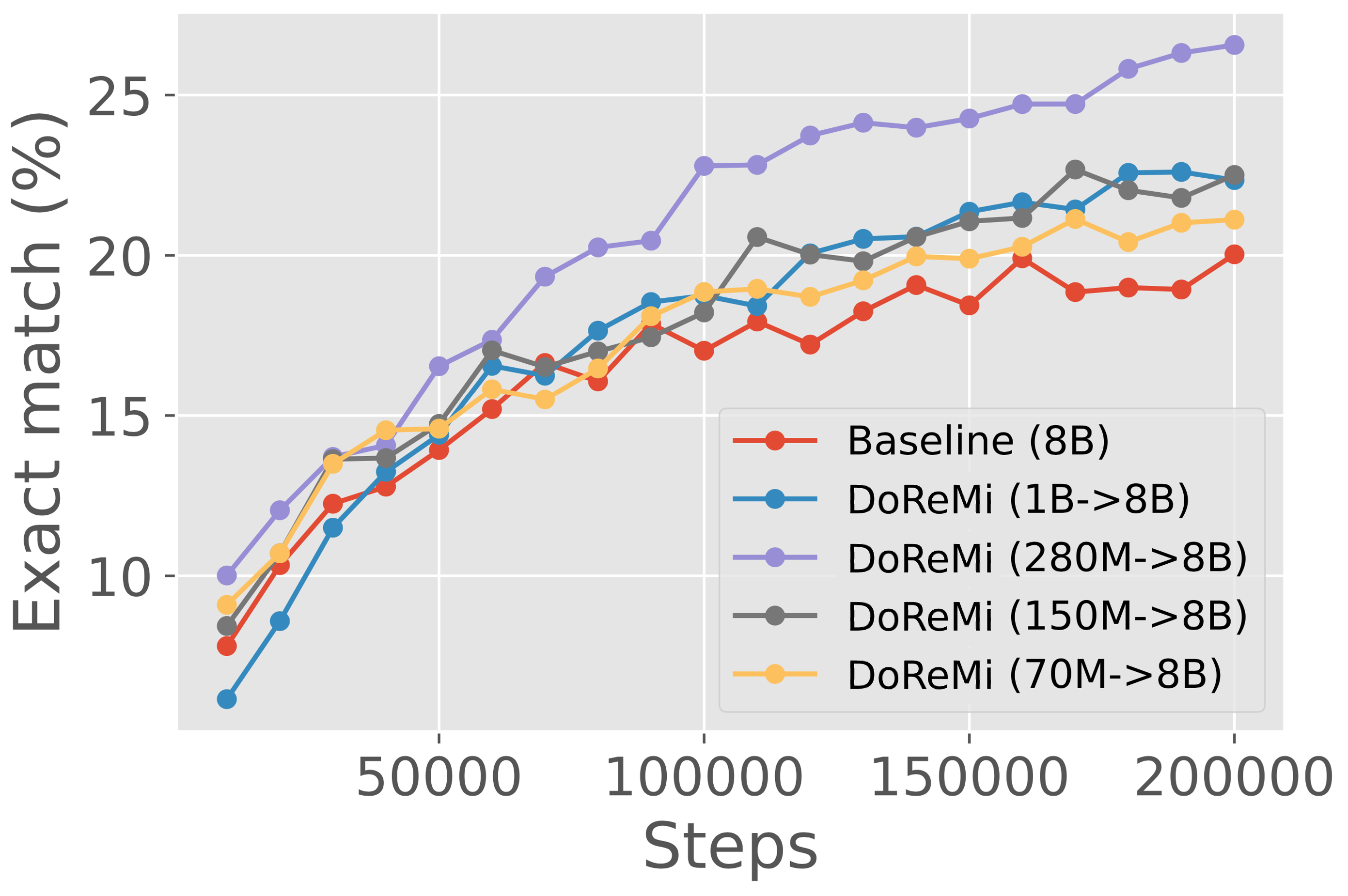}
\end{subfigure}
\hfill
\begin{subfigure}{0.48\textwidth}
\centering
\includegraphics[width=\textwidth]{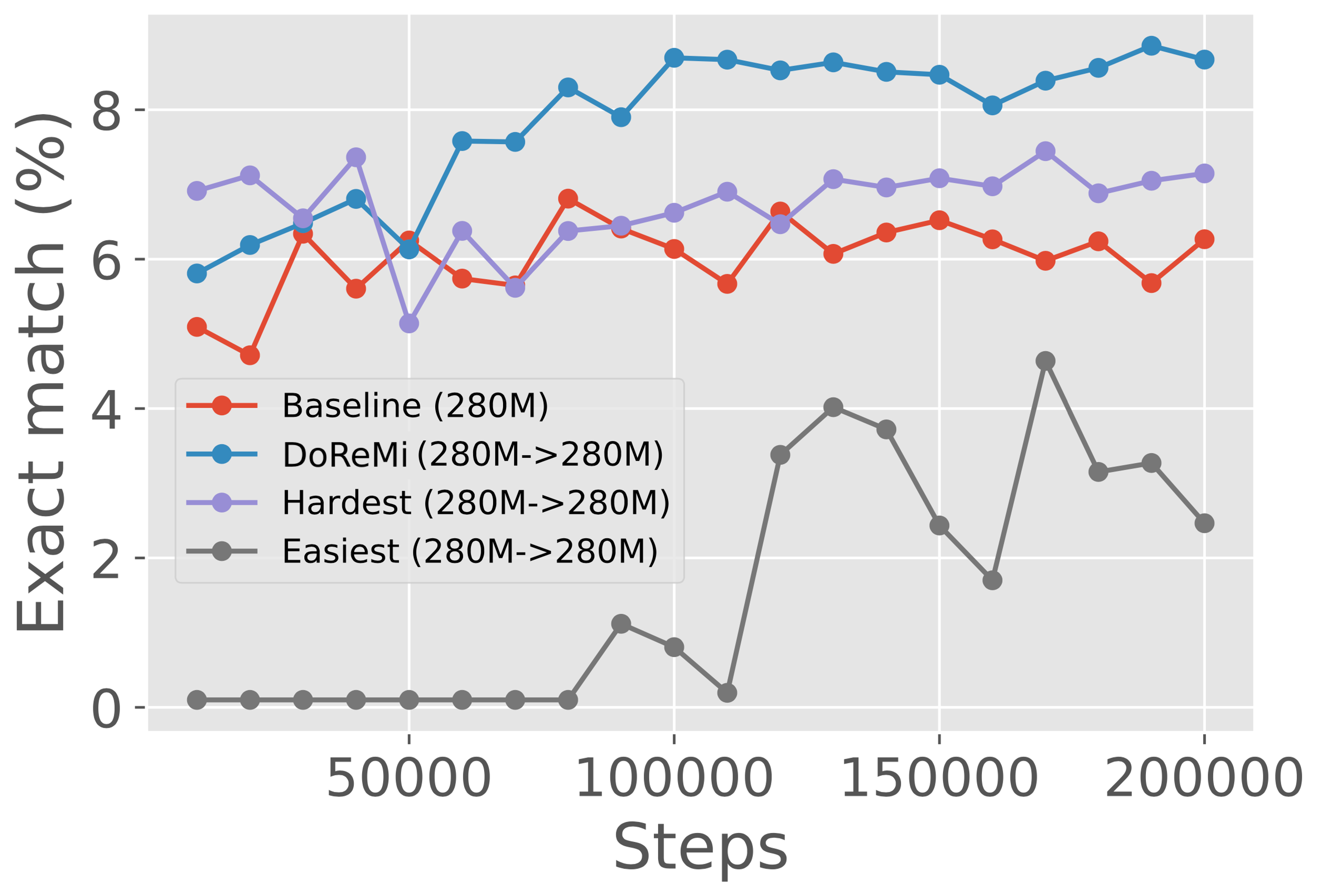}
\end{subfigure}
\caption{Average downstream accuracy for models trained on The Pile. \textbf{(Left)} Increasing the size of the reference/proxy models from 70M to 280M in \algname improves downstream accuracy for a 8B main model, but the trend does not continue for the 1B proxy model. We hypothesize that the Group DRO optimizer is worse for larger proxy models. \textbf{Right)} Optimizing for the hardest or easiest domains rather than excess loss (which combines both) do not achieve the same average downstream accuracy as \algname (280M models). }
\label{fig:pile-ablations}
\end{figure}

\begin{table}[tbp]
\caption{Summary of per-domain log-perplexities on The Pile (22 total domains). Average log-perplexity is an unweighted average of the per-domain log-perplexities.}
\label{tab:perplexity-summaries}
\begin{subfigure}{0.48\textwidth}
\caption{Varying the size of the proxy/reference model and training at 8B.}
\label{tab:pile-8B-scaling}
\centering
\begin{adjustbox}{max width=\textwidth}
\begin{tabular}{lrrr}
\toprule
                & \multirow{2}{*}{\shortstack[r]{Worst-case\\ log-ppl}} & \multirow{2}{*}{\shortstack[r]{Avg log-ppl}} & \multirow{2}{*}{\shortstack[r]{\# domains \\ beating baseline}} \\ 
                & & & \\
                \midrule
Baseline (8B)        & 1.71            & 1.64     & 0/22                          \\
\algname (70M->8B)  & 1.63            & 1.53     & 22/22                         \\
\algname (150M->8B) & 1.56            & 1.52     & 22/22                         \\
\algname (280M->8B) & 1.46            & 1.40      & 22/22                         \\
\algname (1B->8B)   & 1.58            & 1.54     & 22/22                         \\ \bottomrule
\end{tabular}
\end{adjustbox}
\end{subfigure}
\hfill
\begin{subfigure}{0.48\textwidth}
\caption{Perplexity of the \algname main model and proxy model of the same size. Although the 1B proxy model is relatively poor quality, the resulting \weights still improve the main model.
}
\label{tab:pile-perplexity-drolm}
\centering
\begin{adjustbox}{max width=\textwidth}
\begin{tabular}{lrrr}
\toprule
                & \multirow{2}{*}{\shortstack[r]{Worst-case\\ log-ppl}} & \multirow{2}{*}{\shortstack[r]{Avg log-ppl}} & \multirow{2}{*}{\shortstack[r]{\# domains \\ beating baseline}} \\ 
              & & & \\
              \midrule
Baseline (280M)          & 2.39            & 2.32     & 0/22                          \\
\algname (280M->280M) & 2.19            & 2.13     & 22/22                         \\
Proxy (280M)      & 2.33            & 2.27     & 19/22            \\
\midrule
Baseline (1B)      & 1.94            & 1.87     & 0/22                          \\
\algname (1B->1B) & 1.92            & 1.83     & 19/22                         \\
Proxy (1B)  & 2.11            & 2.02     & 0/22   \\
 \bottomrule
\end{tabular}
\end{adjustbox}

\end{subfigure}
\end{table}

\paragraph{Effect of proxy model scale on larger main model's performance.}
We consider 70M, 150M, 280M, and 1B scales for the \algname proxy model while fixing the main model size at 8B (\algname(X$\rightarrow$8B)).
From 70M to 280M, increasing the proxy model size improves downstream accuracy at 8B (Figure~\ref{fig:pile-ablations} left).
We hypothesize that this trend does not continue for the 1B proxy model because the Group DRO optimizer is worse at larger scales (Table~\ref{tab:pile-perplexity-drolm}).
While \algname (280M$\rightarrow$8B) results in the most improvement at 8B, \algname (150M$\rightarrow$8B) and \algname (1B$\rightarrow$8B) still achieve the baseline accuracy almost 2x faster. This suggests that \algname is robust to the proxy model scale. In practice, we suggest choosing a relatively small proxy model size (280M) to save compute.

\paragraph{Choosing the easiest or hardest domains do not suffice.}
We ablate the components of the excess loss metric $\ell_\theta(x) - \ell_{\text{ref}}(x)$  by running \algname using only the loss of the proxy model $\pproxy$ on example $x$, i.e. $\ell_\theta(x)$ (prefer hardest domains for the proxy model) or only the negative loss of the reference $-\ell_{\text{ref}}(x)$ (prefer easiest domains for the reference model). 
Figure~\ref{fig:pile-ablations} (right) shows that neither of the components of the excess loss alone are sufficient to achieve the gains of \algname.
\section{Related Work}
\label{sec:related}

\paragraph{Curating pretraining data for LMs.}
Most closely related is the GLaM dataset~\citep{du2021glam} (also used for training PaLM~\citep{chowdhery2022palm}), which has \weights that are tuned using downstream data.
Optimizing \weights for downstream tasks can be expensive and could require search/zero-order optimization~\citep{snoek12hyper}, RL~\citep{zoph2016neural}, or heuristic assumptions on how positive/negative transfer between domains work.
Example-level filtering also brings benefits for LM training.
The C4 dataset~\citep{raffel2019exploring} shows gains over CommonCrawl via heuristic data cleaning methods.
\citet{du2021glam,xie2023data} show that filtering the data at an example level for high-quality text that look like Wikipedia and books can significantly improve downstream performance for LMs.
In contrast to these works, \algname sets \weights automatically with only two small LM training runs and does not make assumptions about the type of data to prefer (Wikipedia-like, etc.).

\paragraph{General data selection methods.}
Moore-Lewis selection~\citep{moore2010intelligent,axelrod2017cynical,feng2022automatic} selects examples with high cross-entropy difference (similar to excess log-perplexity) between language models trained on target and raw data. In contrast, \algname reweights the data without a target distribution.
\citet{coleman2020selection} select examples based on the uncertainty of a small proxy model for active learning, while \algname uses DRO on the excess loss with respect to a reference model, and focuses on data mixture reweighting.
\citet{mindermann2022prioritized} select examples in an online fashion by taking the top $k$ examples in a minibatch according to excess loss. \algname optimizes the data mixture before training, allowing the larger main model to train in a standard way.
Many other works on data selection are in vision~\citep{sorscher2022beyond,kaushal2019learning,killamsetty2021glister,killamsetty2021gradmatch,killamsetty2021retrieve,wang2020optimizing,wei2015submodular,paul2021diet,mirzasoleiman2020coresets,sener2018active} and mainly focus on example-level subset selection with metrics such as gradient matching.
Overall, these methods do not address data selection for pretraining, where the downstream data distribution may be very different from the pretraining distribution. \algname aims to address the pretraining/downstream distribution shift with a robust optimization approach. To the best of our knowledge, we are the first to show that reweighting the data according to losses of a small proxy LM can improve the training efficiency of much larger LM.

\paragraph{Distributionally robust optimization.}
Within DRO methods for deep learning~\citep{bental2013robust,sinha2018certifiable,oren2019drolm,sagawa2020group}, we target a restricted form of shift called group shifts~\citep{duchi2019distributionally,oren2019drolm,sagawa2020group}, where the test distribution can be an unknown mixture of groups (domains).
We follow DRO-LM~\citep{oren2019drolm}, which employs DRO for LMs in the group shift setting.
DRO-LM also uses a baselined loss, but with a simple bigram reference model.
\algname uses a reference model of the same size and architecture as the proxy model to ensure that the losses are on a similar scale.
During optimization, DRO-LM takes a worst-case subset of each minibatch to update the model on, while we use the Group DRO optimizer~\citep{sagawa2020group} which doesn't require online subselection.
If we equalize the number of examples in each minibatch used for gradient updates, online subselelction is more expensive than Group DRO since it requires running forward passes on a larger minibatch (e.g., double the minibatch size) before selecting a subset to update the model with.
In comparison, the Group DRO optimizer updates the model on all examples in a weighted fashion.
Overall, in contrast to these DRO methods which aim to produce robust \textbf{models}, we use DRO to optimize the \textbf{data} for training larger models more efficiently.

\paragraph{Data-centric AI.}
Large-scale datasets and benchmarks have driven much of the recent progress in AI, including vision, NLP, and multimodal models~\citep{deng2009imagenet,russakovsky2015imagenet,wang2019glue,rajpurkar2016squad,raffel2019exploring, gao2020pile,schuhmann2022laion5b,gadre2023datacomp}. 
However, most datasets are still painstakingly created with human-generated data, manual work, and heuristics~\citep{deng2009imagenet,raffel2019exploring,gao2020pile,schuhmann2022laion5b,gadre2023datacomp}.
\algname is a principled data-centric method that aims to improve language model training efficiency.
We hope that \algname can provide a starting point for a general data-centric framework for language modeling via robust optimization.
\section{Discussion and Limitations}
\label{sec:discussion}

\paragraph{Saving compute in \algname with extrapolation.}
In Section~\ref{sec:method}, we run \algname for the number of training steps that will be used to train the final model, which could be unnecessarily expensive.
A future direction for saving compute would be to stop running \algname at an early step and extrapolate the \weights for the desired number of steps, since we found that most of the variation in the \weights during a \algname run seems to occur in the beginning of training (Appendix Figure~\ref{fig:pile-weights-evolution}).

\paragraph{Choice of reference model.}
The choice of reference model can affect the \weights found by \algname. For example, iterated \algname (Section~\ref{sec:experiments}) improves performance by using a reference model trained on the tuned \weights from a previous round of \algname.
Further directions include varying the reference model size and using specialized reference models to optimize \weights for a specific application area.

\paragraph{What is a domain?}
We define a domain by data provenance in our experiments, but this only enables coarse-grained control. Using fine-grained domains could improve the gains from \algname. For example, \algname is more effective on The Pile (22 domains) than the GLaM dataset (8 domains).
Open directions include automatically finding fine-grained domains (e.g., via clustering as in DRO-LM~\citep{oren2019drolm}) and reweighting the data at an example level. 
When domains are very fine-grained, it will be important to control the pessimism of DRO (e.g., DRO can put all the weight on a small set of worst-case examples).

\paragraph{Transferability of \weights across scales.}
We optimized the \weights with a small proxy model (280M) and directly used these \weights to improve training at a larger scale (8B).
Understanding why the \weights can be transferred across scales and the limits of how far these \weights transfer are important questions to answer in future work.

\paragraph{Broader impacts.}
Large language models are 
We hope to improve training efficiency and reduce the environmental impact of training large LMs~\citep{strubell2019energy,lacoste2019quantifying,patterson2021carbon,ligozat2021unraveling}.
In particular, by reducing the training time by 2x, we can halve the cost and energy consumption of training large language models.
Since such efficiency improvements may be used to develop even larger models, there may be no absolute improvement in energy consumption.
Ultimately, we hope to improve the training efficiency and cost of developing future language models relative to existing methods.

Large LMs have also been well-documented to have risks and biases~\citep{abid2021persistent,nadeem2020stereoset,bommasani2021opportunities,blodgett2017racial,gehman2020realtoxicityprompts}.
For example, GPT-3 tends to have an anti-Muslim bias, where Muslims are frequently related to violence or terrorism in analogy and completion tasks~\citep{abid2021persistent}.
As large language models are increasingly relied upon in applications, the magnitude of the risks increases~\citep{bommasani2022homogenization}.
Distributionally robust optimization (DRO), which is used in \algname to optimize the data mixture, can have a favorable impact on fairness~\citep{hashimoto2018repeated}.
While the standard approach of minimizing the average loss can lead to disparate performance on minority subgroups that do not contribute heavily to the loss~\citep{amodei2016}, DRO promotes good performance on all groups via a worst-case loss.
In this way, DRO-style data-centric methods such as \algname can improve the representation disparity between majority and minority subgroups in a dataset.

\section{Conclusion}
\label{sec:conclusion}
We introduced \algname, an algorithm reweighting data domains for training language models.
\algname is able to run on small models and transfer the benefits to 30x larger models, resulting in a 2.6x speedup in training on the Pile just by changing the sampling probabilities on domains.
We hope to instigate more research on data-centric approaches for improving language model training efficiency.

\section*{Acknowledgments}
We thank Xiangning Chen, Andrew Dai, Zoubin Ghahramani, Balaji Lakshminarayanan, Paul Michel, Yonghui Wu, Steven Zheng, Chen Zhu and the broader Google Bard team members for insightful discussions and pointers.

\bibliography{main}

\begin{thebibliography}{56}
\providecommand{\natexlab}[1]{#1}
\providecommand{\url}[1]{\texttt{#1}}
\expandafter\ifx\csname urlstyle\endcsname\relax
  \providecommand{\doi}[1]{doi: #1}\else
  \providecommand{\doi}{doi: \begingroup \urlstyle{rm}\Url}\fi

\bibitem[Abid et~al.(2021)Abid, Farooqi, and Zou]{abid2021persistent}
Abubakar Abid, Maheen Farooqi, and James Zou.
\newblock Persistent anti-muslim bias in large language models.
\newblock \emph{arXiv preprint arXiv:2101.05783}, 2021.

\bibitem[Amodei et~al.(2016)]{amodei2016}
Dario Amodei et~al.
\newblock Deep speech 2 end to end speech recognition in {E}nglish and
  mandarin.
\newblock In \emph{International Conference on Machine Learning (ICML)}, pages
  173--182, 2016.

\bibitem[Axelrod(2017)]{axelrod2017cynical}
Amittai Axelrod.
\newblock Cynical selection of language model training data.
\newblock \emph{CoRR}, abs/1709.02279, 2017.
\newblock URL \url{http://arxiv.org/abs/1709.02279}.

\bibitem[Ben-Tal et~al.(2013)Ben-Tal, den Hertog, Waegenaere, Melenberg, and
  Rennen]{bental2013robust}
Aharon Ben-Tal, Dick den Hertog, Anja~De Waegenaere, Bertrand Melenberg, and
  Gijs Rennen.
\newblock Robust solutions of optimization problems affected by uncertain
  probabilities.
\newblock \emph{Management Science}, 59:\penalty0 341--357, 2013.

\bibitem[Berant et~al.(2013)Berant, Chou, Frostig, and
  Liang]{berant2013freebase}
Jonathan Berant, Andrew Chou, Roy Frostig, and Percy Liang.
\newblock Semantic parsing on {F}reebase from question-answer pairs.
\newblock In \emph{Empirical Methods in Natural Language Processing (EMNLP)},
  2013.

\bibitem[Blodgett and OConnor(2017)]{blodgett2017racial}
Su~Lin Blodgett and Brendan OConnor.
\newblock Racial disparity in natural language processing: A case study of
  social media {A}frican-{A}merican {E}nglish.
\newblock \emph{arXiv preprint arXiv:1707.00061}, 2017.

\bibitem[Bommasani et~al.(2021)Bommasani, Hudson, Adeli, Altman, Arora, von
  Arx, Bernstein, Bohg, Bosselut, Brunskill, Brynjolfsson, Buch, Card,
  Castellon, Chatterji, Chen, Creel, Davis, Demszky, Donahue, Doumbouya,
  Durmus, Ermon, Etchemendy, Ethayarajh, Fei-Fei, Finn, Gale, Gillespie, Goel,
  Goodman, Grossman, Guha, Hashimoto, Henderson, Hewitt, Ho, Hong, Hsu, Huang,
  Icard, Jain, Jurafsky, Kalluri, Karamcheti, Keeling, Khani, Khattab, Koh,
  Krass, Krishna, Kuditipudi, Kumar, Ladhak, Lee, Lee, Leskovec, Levent, Li,
  Li, Ma, Malik, Manning, Mirchandani, Mitchell, Munyikwa, Nair, Narayan,
  Narayanan, Newman, Nie, Niebles, Nilforoshan, Nyarko, Ogut, Orr,
  Papadimitriou, Park, Piech, Portelance, Potts, Raghunathan, Reich, Ren, Rong,
  Roohani, Ruiz, Ryan, Ré, Sadigh, Sagawa, Santhanam, Shih, Srinivasan,
  Tamkin, Taori, Thomas, Tramèr, Wang, Wang, Wu, Wu, Wu, Xie, Yasunaga, You,
  Zaharia, Zhang, Zhang, Zhang, Zhang, Zheng, Zhou, and
  Liang]{bommasani2021opportunities}
Rishi Bommasani, Drew~A. Hudson, Ehsan Adeli, Russ Altman, Simran Arora, Sydney
  von Arx, Michael~S. Bernstein, Jeannette Bohg, Antoine Bosselut, Emma
  Brunskill, Erik Brynjolfsson, Shyamal Buch, Dallas Card, Rodrigo Castellon,
  Niladri Chatterji, Annie Chen, Kathleen Creel, Jared~Quincy Davis, Dorottya
  Demszky, Chris Donahue, Moussa Doumbouya, Esin Durmus, Stefano Ermon, John
  Etchemendy, Kawin Ethayarajh, Li~Fei-Fei, Chelsea Finn, Trevor Gale, Lauren
  Gillespie, Karan Goel, Noah Goodman, Shelby Grossman, Neel Guha, Tatsunori
  Hashimoto, Peter Henderson, John Hewitt, Daniel~E. Ho, Jenny Hong, Kyle Hsu,
  Jing Huang, Thomas Icard, Saahil Jain, Dan Jurafsky, Pratyusha Kalluri,
  Siddharth Karamcheti, Geoff Keeling, Fereshte Khani, Omar Khattab, Pang~Wei
  Koh, Mark Krass, Ranjay Krishna, Rohith Kuditipudi, Ananya Kumar, Faisal
  Ladhak, Mina Lee, Tony Lee, Jure Leskovec, Isabelle Levent, Xiang~Lisa Li,
  Xuechen Li, Tengyu Ma, Ali Malik, Christopher~D. Manning, Suvir Mirchandani,
  Eric Mitchell, Zanele Munyikwa, Suraj Nair, Avanika Narayan, Deepak
  Narayanan, Ben Newman, Allen Nie, Juan~Carlos Niebles, Hamed Nilforoshan,
  Julian Nyarko, Giray Ogut, Laurel Orr, Isabel Papadimitriou, Joon~Sung Park,
  Chris Piech, Eva Portelance, Christopher Potts, Aditi Raghunathan, Rob Reich,
  Hongyu Ren, Frieda Rong, Yusuf Roohani, Camilo Ruiz, Jack Ryan, Christopher
  Ré, Dorsa Sadigh, Shiori Sagawa, Keshav Santhanam, Andy Shih, Krishnan
  Srinivasan, Alex Tamkin, Rohan Taori, Armin~W. Thomas, Florian Tramèr,
  Rose~E. Wang, William Wang, Bohan Wu, Jiajun Wu, Yuhuai Wu, Sang~Michael Xie,
  Michihiro Yasunaga, Jiaxuan You, Matei Zaharia, Michael Zhang, Tianyi Zhang,
  Xikun Zhang, Yuhui Zhang, Lucia Zheng, Kaitlyn Zhou, and Percy Liang.
\newblock On the opportunities and risks of foundation models.
\newblock \emph{arXiv preprint arXiv:2108.07258}, 2021.

\bibitem[Bommasani et~al.(2022)Bommasani, Creel, Kumar, Jurafsky, and
  Liang]{bommasani2022homogenization}
Rishi Bommasani, Kathleen~A. Creel, Ananya Kumar, Dan Jurafsky, and Percy
  Liang.
\newblock Picking on the same person: Does algorithmic monoculture lead to
  outcome homogenization?
\newblock In \emph{Advances in Neural Information Processing Systems
  (NeurIPS)}, 2022.

\bibitem[Brown et~al.(2020)Brown, Mann, Ryder, Subbiah, Kaplan, Dhariwal,
  Neelakantan, Shyam, Sastry, Askell, Agarwal, Herbert-Voss, Krueger, Henighan,
  Child, Ramesh, Ziegler, Wu, Winter, Hesse, Chen, Sigler, Litwin, Gray, Chess,
  Clark, Berner, McCandlish, Radford, Sutskever, and Amodei]{brown2020gpt3}
Tom~B. Brown, Benjamin Mann, Nick Ryder, Melanie Subbiah, Jared Kaplan,
  Prafulla Dhariwal, Arvind Neelakantan, Pranav Shyam, Girish Sastry, Amanda
  Askell, Sandhini Agarwal, Ariel Herbert-Voss, Gretchen Krueger, Tom Henighan,
  Rewon Child, Aditya Ramesh, Daniel~M. Ziegler, Jeffrey Wu, Clemens Winter,
  Christopher Hesse, Mark Chen, Eric Sigler, Mateusz Litwin, Scott Gray,
  Benjamin Chess, Jack Clark, Christopher Berner, Sam McCandlish, Alec Radford,
  Ilya Sutskever, and Dario Amodei.
\newblock Language models are few-shot learners.
\newblock \emph{arXiv preprint arXiv:2005.14165}, 2020.

\bibitem[Chowdhery et~al.(2022)Chowdhery, Narang, Devlin, Bosma, Mishra,
  Roberts, Barham, Chung, Sutton, Gehrmann, Schuh, Shi, Tsvyashchenko, Maynez,
  Rao, Barnes, Tay, Shazeer, Prabhakaran, Reif, Du, Hutchinson, Pope, Bradbury,
  Austin, Isard, Gur-Ari, Yin, Duke, Levskaya, Ghemawat, Dev, Michalewski,
  García, Misra, Robinson, Fedus, Zhou, Ippolito, Luan, Lim, Zoph, Spiridonov,
  Sepassi, Dohan, Agrawal, Omernick, Dai, Pillai, Pellat, Lewkowycz, Moreira,
  Child, Polozov, Lee, Zhou, Wang, Saeta, Diaz, Firat, Catasta, Wei,
  Meier-Hellstern, Eck, Dean, Petrov, and Fiedel]{chowdhery2022palm}
Aakanksha Chowdhery, Sharan Narang, Jacob Devlin, Maarten Bosma, Gaurav Mishra,
  Adam Roberts, Paul Barham, Hyung~Won Chung, Charles Sutton, Sebastian
  Gehrmann, Parker Schuh, Kensen Shi, Sasha Tsvyashchenko, Joshua Maynez,
  A.~Rao, Parker Barnes, Yi~Tay, Noam~M. Shazeer, Vinodkumar Prabhakaran, Emily
  Reif, Nan Du, B.~Hutchinson, Reiner Pope, James Bradbury, Jacob Austin,
  M.~Isard, Guy Gur-Ari, Pengcheng Yin, Toju Duke, Anselm Levskaya,
  S.~Ghemawat, Sunipa Dev, Henryk Michalewski, Xavier García, Vedant Misra,
  Kevin Robinson, Liam Fedus, Denny Zhou, Daphne Ippolito, D.~Luan, Hyeontaek
  Lim, Barret Zoph, A.~Spiridonov, Ryan Sepassi, David Dohan, Shivani Agrawal,
  Mark Omernick, Andrew~M. Dai, T.~S. Pillai, Marie Pellat, Aitor Lewkowycz,
  E.~Moreira, Rewon Child, Oleksandr Polozov, Katherine Lee, Zongwei Zhou,
  Xuezhi Wang, Brennan Saeta, Mark Diaz, Orhan Firat, Michele Catasta, Jason
  Wei, K.~Meier-Hellstern, D.~Eck, J.~Dean, Slav Petrov, and Noah Fiedel.
\newblock {PaLM}: Scaling language modeling with pathways.
\newblock \emph{arXiv}, 2022.

\bibitem[Coleman et~al.(2020)Coleman, Yeh, Mussmann, Mirzasoleiman, Bailis,
  Liang, Leskovec, and Zaharia]{coleman2020selection}
Cody Coleman, Christopher Yeh, Stephen Mussmann, Baharan Mirzasoleiman, Peter
  Bailis, Percy Liang, Jure Leskovec, and Matei Zaharia.
\newblock Selection via proxy: Efficient data selection for deep learning.
\newblock In \emph{International Conference on Learning Representations
  (ICLR)}, 2020.

\bibitem[Deng et~al.(2009)Deng, Dong, Socher, Li, Li, and
  Fei-Fei]{deng2009imagenet}
Jia Deng, Wei Dong, Richard Socher, Li-Jia Li, Kai Li, and Li~Fei-Fei.
\newblock {I}mage{N}et: A large-scale hierarchical image database.
\newblock In \emph{Computer Vision and Pattern Recognition (CVPR)}, pages
  248--255, 2009.

\bibitem[Du et~al.(2021)Du, Huang, Dai, Tong, Lepikhin, Xu, Krikun, Zhou, Yu,
  Firat, Zoph, Fedus, Bosma, Zhou, Wang, Wang, Webster, Pellat, Robinson,
  Meier-Hellstern, Duke, Dixon, Zhang, Le, Wu, Chen, and Cui]{du2021glam}
Nan Du, Yanping Huang, Andrew~M. Dai, Simon Tong, Dmitry Lepikhin, Yuanzhong
  Xu, M.~Krikun, Yanqi Zhou, Adams~Wei Yu, Orhan Firat, Barret Zoph, Liam
  Fedus, Maarten Bosma, Zongwei Zhou, Tao Wang, Yu~Emma Wang, Kellie Webster,
  Marie Pellat, Kevin Robinson, K.~Meier-Hellstern, Toju Duke, Lucas Dixon, Kun
  Zhang, Quoc~V. Le, Yonghui Wu, Zhifeng Chen, and Claire Cui.
\newblock {GLaM}: Efficient scaling of language models with mixture-of-experts.
\newblock \emph{arXiv}, 2021.

\bibitem[Duchi et~al.(2019)Duchi, Hashimoto, and
  Namkoong]{duchi2019distributionally}
John Duchi, Tatsunori Hashimoto, and Hongseok Namkoong.
\newblock Distributionally robust losses against mixture covariate shifts.
\newblock
  \url{https://cs.stanford.edu/~thashim/assets/publications/condrisk.pdf},
  2019.

\bibitem[Feng et~al.(2022)Feng, Xia, Van~Durme, and Sedoc]{feng2022automatic}
Yukun Feng, Patrick Xia, Benjamin Van~Durme, and João Sedoc.
\newblock Automatic document selection for efficient encoder pretraining, 2022.
\newblock URL \url{https://arxiv.org/abs/2210.10951}.

\bibitem[Gadre et~al.(2023)Gadre, Ilharco, Fang, Hayase, Smyrnis, Nguyen,
  Marten, Wortsman, Ghosh, Zhang, Orgad, Entezari, Daras, Pratt, Ramanujan,
  Bitton, Marathe, Mussmann, Vencu, Cherti, Krishna, Koh, Saukh, Ratner, Song,
  Hajishirzi, Farhadi, Beaumont, Oh, Dimakis, Jitsev, Carmon, Shankar, and
  Schmidt]{gadre2023datacomp}
Samir~Yitzhak Gadre, Gabriel Ilharco, Alex Fang, Jonathan Hayase, Georgios
  Smyrnis, Thao Nguyen, Ryan Marten, Mitchell Wortsman, Dhruba Ghosh, Jieyu
  Zhang, Eyal Orgad, Rahim Entezari, Giannis Daras, Sarah Pratt, Vivek
  Ramanujan, Yonatan Bitton, Kalyani Marathe, Stephen Mussmann, Richard Vencu,
  Mehdi Cherti, Ranjay Krishna, Pang~Wei Koh, Olga Saukh, Alexander Ratner,
  Shuran Song, Hannaneh Hajishirzi, Ali Farhadi, Romain Beaumont, Sewoong Oh,
  Alex Dimakis, Jenia Jitsev, Yair Carmon, Vaishaal Shankar, and Ludwig
  Schmidt.
\newblock Datacomp: In search of the next generation of multimodal datasets.
\newblock \emph{arXiv preprint arXiv:2304.14108}, 2023.

\bibitem[Gao et~al.(2020)Gao, Biderman, Black, Golding, Hoppe, Foster, Phang,
  He, Thite, Nabeshima, Presser, and Leahy]{gao2020pile}
Leo Gao, Stella Biderman, Sid Black, Laurence Golding, Travis Hoppe, Charles
  Foster, Jason Phang, Horace He, Anish Thite, Noa Nabeshima, Shawn Presser,
  and Connor Leahy.
\newblock The pile: An 800gb dataset of diverse text for language modeling.
\newblock \emph{arXiv}, 2020.

\bibitem[Gehman et~al.(2020)Gehman, Gururangan, Sap, Choi, and
  Smith]{gehman2020realtoxicityprompts}
Samuel Gehman, Suchin Gururangan, Maarten Sap, Yejin Choi, and Noah~A Smith.
\newblock Realtoxicityprompts: Evaluating neural toxic degeneration in language
  models.
\newblock \emph{arXiv preprint arXiv:2009.11462}, 2020.

\bibitem[Hashimoto et~al.(2018)Hashimoto, Srivastava, Namkoong, and
  Liang]{hashimoto2018repeated}
Tatsunori~B. Hashimoto, Megha Srivastava, Hongseok Namkoong, and Percy Liang.
\newblock Fairness without demographics in repeated loss minimization.
\newblock In \emph{International Conference on Machine Learning (ICML)}, 2018.

\bibitem[Hoffmann et~al.(2022)Hoffmann, Borgeaud, Mensch, Buchatskaya, Cai,
  Rutherford, de~Las~Casas, Hendricks, Welbl, Clark, Hennigan, Noland,
  Millican, van~den Driessche, Damoc, Guy, Osindero, Simonyan, Elsen, Rae,
  Vinyals, and Sifre]{hoffmann2022chinchilla}
Jordan Hoffmann, Sebastian Borgeaud, Arthur Mensch, Elena Buchatskaya, Trevor
  Cai, Eliza Rutherford, Diego de~Las~Casas, Lisa~Anne Hendricks, Johannes
  Welbl, Aidan Clark, Tom Hennigan, Eric Noland, Katie Millican, George van~den
  Driessche, Bogdan Damoc, Aurelia Guy, Simon Osindero, Karen Simonyan, Erich
  Elsen, Jack~W. Rae, Oriol Vinyals, and Laurent Sifre.
\newblock An empirical analysis of compute-optimal large language model
  training.
\newblock In \emph{Advances in Neural Information Processing Systems
  (NeurIPS)}, 2022.

\bibitem[Joshi et~al.(2017)Joshi, Choi, Weld, and
  Zettlemoyer]{joshi2017triviaqa}
Mandar Joshi, Eunsol Choi, Daniel Weld, and Luke Zettlemoyer.
\newblock {TriviaQA}: A large scale distantly supervised challenge dataset for
  reading comprehension.
\newblock In \emph{Association for Computational Linguistics (ACL)}, 2017.

\bibitem[Kaushal et~al.(2019)Kaushal, Iyer, Kothawade, Mahadev, Doctor, and
  Ramakrishnan]{kaushal2019learning}
Vishal Kaushal, Rishabh Iyer, Suraj Kothawade, Rohan Mahadev, Khoshrav Doctor,
  and Ganesh Ramakrishnan.
\newblock Learning from less data: A unified data subset selection and active
  learning framework for computer vision.
\newblock \emph{IEEE/CVF Winter Conference on Applicatios of Computer Vision
  (WACV)}, 2019.

\bibitem[Killamsetty et~al.(2021{\natexlab{a}})Killamsetty, S, Ramakrishnan,
  De, and Iyer]{killamsetty2021gradmatch}
Krishnateja Killamsetty, Durga S, Ganesh Ramakrishnan, Abir De, and Rishabh
  Iyer.
\newblock {GRAD-MATCH}: Gradient matching based data subset selection for
  efficient deep model training.
\newblock In \emph{International Conference on Machine Learning (ICML)},
  2021{\natexlab{a}}.

\bibitem[Killamsetty et~al.(2021{\natexlab{b}})Killamsetty, Sivasubramanian,
  Ramakrishnan, and Iyer]{killamsetty2021glister}
Krishnateja Killamsetty, Durga Sivasubramanian, Ganesh Ramakrishnan, and
  Rishabh Iyer.
\newblock Glister: Generalization based data subset selection for efficient and
  robust learning.
\newblock In \emph{Association for the Advancement of Artificial Intelligence
  (AAAI)}, 2021{\natexlab{b}}.

\bibitem[Killamsetty et~al.(2021{\natexlab{c}})Killamsetty, Zhao, Chen, and
  Iyer]{killamsetty2021retrieve}
Krishnateja Killamsetty, Xujiang Zhao, Feng Chen, and Rishabh Iyer.
\newblock Retrieve: Coreset selection for efficient and robust semi-supervised
  learning.
\newblock In \emph{Advances in Neural Information Processing Systems
  (NeurIPS)}, 2021{\natexlab{c}}.

\bibitem[Kingma and Ba(2015)]{kingma2015adam}
Diederik Kingma and Jimmy Ba.
\newblock Adam: A method for stochastic optimization.
\newblock In \emph{International Conference on Learning Representations
  (ICLR)}, 2015.

\bibitem[Kwiatkowski et~al.(2019)Kwiatkowski, Palomaki, Redfield, Collins,
  Parikh, Alberti, Epstein, Polosukhin, Kelcey, Devlin, Lee, Toutanova, Jones,
  Chang, Dai, Uszkoreit, Le, and Petrov]{kwiatkowski2019natural}
Tom Kwiatkowski, Jennimaria Palomaki, Olivia Redfield, Michael Collins, Ankur
  Parikh, Chris Alberti, Danielle Epstein, Illia Polosukhin, Matthew Kelcey,
  Jacob Devlin, Kenton Lee, Kristina~N. Toutanova, Llion Jones, Ming-Wei Chang,
  Andrew Dai, Jakob Uszkoreit, Quoc Le, and Slav Petrov.
\newblock Natural questions: A benchmark for question answering research.
\newblock In \emph{Association for Computational Linguistics (ACL)}, 2019.

\bibitem[Lacoste et~al.(2019)Lacoste, Luccioni, Schmidt, and
  Dandres]{lacoste2019quantifying}
Alexandre Lacoste, Alexandra Luccioni, Victor Schmidt, and Thomas Dandres.
\newblock Quantifying the carbon emissions of machine learning.
\newblock \emph{arXiv preprint arXiv:1910.09700}, 2019.

\bibitem[Ligozat et~al.(2021)Ligozat, Lef{\`{e}}vre, Bugeau, and
  Combaz]{ligozat2021unraveling}
Anne{-}Laure Ligozat, Julien Lef{\`{e}}vre, Aur{\'{e}}lie Bugeau, and Jacques
  Combaz.
\newblock Unraveling the hidden environmental impacts of {AI} solutions for
  environment.
\newblock \emph{CoRR}, abs/2110.11822, 2021.
\newblock URL \url{https://arxiv.org/abs/2110.11822}.

\bibitem[Mindermann et~al.(2022)Mindermann, Brauner, Razzak, Sharma, Kirsch,
  Xu, Höltgen, Gomez, Morisot, Farquhar, and Gal]{mindermann2022prioritized}
Sören Mindermann, Jan Brauner, Muhammed Razzak, Mrinank Sharma, Andreas
  Kirsch, Winnie Xu, Benedikt Höltgen, Aidan~N. Gomez, Adrien Morisot,
  Sebastian Farquhar, and Yarin Gal.
\newblock Prioritized training on points that are learnable, worth learning,
  and not yet learnt.
\newblock In \emph{International Conference on Machine Learning (ICML)}, 2022.

\bibitem[Mirzasoleiman et~al.(2020)Mirzasoleiman, Bilmes, and
  Leskovec]{mirzasoleiman2020coresets}
Baharan Mirzasoleiman, Jeff Bilmes, and Jure Leskovec.
\newblock Coresets for data-efficient training of machine learning models.
\newblock In \emph{International Conference on Machine Learning (ICML)}, 2020.

\bibitem[Moore and Lewis(2010)]{moore2010intelligent}
Robert~C. Moore and William Lewis.
\newblock Intelligent selection of language model training data.
\newblock In \emph{Proceedings of the {ACL} 2010 Conference Short Papers},
  pages 220--224, Uppsala, Sweden, July 2010. Association for Computational
  Linguistics.
\newblock URL \url{https://aclanthology.org/P10-2041}.

\bibitem[Nadeem et~al.(2020)Nadeem, Bethke, and Reddy]{nadeem2020stereoset}
Moin Nadeem, Anna Bethke, and Siva Reddy.
\newblock Stereoset: Measuring stereotypical bias in pretrained language
  models.
\newblock \emph{arXiv preprint arXiv:2004.09456}, 2020.

\bibitem[Nemirovski et~al.(2009)Nemirovski, Juditsky, Lan, and
  Shapiro]{nemirovski2009robust}
Arkadi Nemirovski, Anatoli Juditsky, Guanghui Lan, and Alexander Shapiro.
\newblock Robust stochastic approximation approach to stochastic programming.
\newblock \emph{SIAM Journal on optimization}, 19\penalty0 (4):\penalty0
  1574--1609, 2009.

\bibitem[Oren et~al.(2019)Oren, Sagawa, Hashimoto, and Liang]{oren2019drolm}
Yonatan Oren, Shiori Sagawa, Tatsunori Hashimoto, and Percy Liang.
\newblock Distributionally robust language modeling.
\newblock In \emph{Empirical Methods in Natural Language Processing (EMNLP)},
  2019.

\bibitem[Paperno et~al.(2016)Paperno, Kruszewski, Lazaridou, Pham, Bernardi,
  Pezzelle, Baroni, Boleda, and Fernandez]{paperno2016lambada}
Denis Paperno, German Kruszewski, Angeliki Lazaridou, Quan~Ngoc Pham, Raffaella
  Bernardi, Sandro Pezzelle, Marco Baroni, Gemma Boleda, and Raquel Fernandez.
\newblock The {LAMBADA} dataset: Word prediction requiring a broad discourse
  context.
\newblock In \emph{Association for Computational Linguistics (ACL)}, 2016.

\bibitem[Patterson et~al.(2021)Patterson, Gonzalez, Le, Liang, Munguia,
  Rothchild, So, Texier, and Dean]{patterson2021carbon}
David~A. Patterson, Joseph Gonzalez, Quoc~V. Le, Chen Liang, Lluis{-}Miquel
  Munguia, Daniel Rothchild, David~R. So, Maud Texier, and Jeff Dean.
\newblock Carbon emissions and large neural network training.
\newblock \emph{CoRR}, abs/2104.10350, 2021.
\newblock URL \url{https://arxiv.org/abs/2104.10350}.

\bibitem[Paul et~al.(2021)Paul, Ganguli, and Dziugaite]{paul2021diet}
Mansheej Paul, Surya Ganguli, and Gintare~Karolina Dziugaite.
\newblock Deep learning on a data diet: Finding important examples early in
  training.
\newblock In \emph{Association for the Advancement of Artificial Intelligence
  (AAAI)}, 2021.

\bibitem[Raffel et~al.(2019)Raffel, Shazeer, Roberts, Lee, Narang, Matena,
  Zhou, Li, and Liu]{raffel2019exploring}
Colin Raffel, Noam Shazeer, Adam Roberts, Katherine Lee, Sharan Narang, Michael
  Matena, Yanqi Zhou, Wei Li, and Peter~J. Liu.
\newblock Exploring the limits of transfer learning with a unified text-to-text
  transformer.
\newblock \emph{arXiv preprint arXiv:1910.10683}, 2019.

\bibitem[Rajpurkar et~al.(2016)Rajpurkar, Zhang, Lopyrev, and
  Liang]{rajpurkar2016squad}
Pranav Rajpurkar, Jian Zhang, Konstantin Lopyrev, and Percy Liang.
\newblock {SQuAD}: 100,000+ questions for machine comprehension of text.
\newblock In \emph{Empirical Methods in Natural Language Processing (EMNLP)},
  2016.

\bibitem[Rajpurkar et~al.(2018)Rajpurkar, Jia, and
  Liang]{rajpurkar2018squadrun}
Pranav Rajpurkar, Robin Jia, and Percy Liang.
\newblock Know what you don't know: Unanswerable questions for {SQuAD}.
\newblock In \emph{Association for Computational Linguistics (ACL)}, 2018.

\bibitem[Russakovsky et~al.(2015)Russakovsky, Deng, Su, Krause, Satheesh, Ma,
  Huang, Karpathy, Khosla, Bernstein, et~al.]{russakovsky2015imagenet}
Olga Russakovsky, Jia Deng, Hao Su, Jonathan Krause, Sanjeev Satheesh, Sean Ma,
  Zhiheng Huang, Andrej Karpathy, Aditya Khosla, Michael Bernstein, et~al.
\newblock {ImageNet} large scale visual recognition challenge.
\newblock \emph{International Journal of Computer Vision}, 115\penalty0
  (3):\penalty0 211--252, 2015.

\bibitem[Sagawa et~al.(2020)Sagawa, Koh, Hashimoto, and Liang]{sagawa2020group}
Shiori Sagawa, Pang~Wei Koh, Tatsunori~B. Hashimoto, and Percy Liang.
\newblock Distributionally robust neural networks for group shifts: On the
  importance of regularization for worst-case generalization.
\newblock In \emph{International Conference on Learning Representations
  (ICLR)}, 2020.

\bibitem[Schuhmann et~al.(2022)Schuhmann, Beaumont, Vencu, Gordon, Wightman,
  Cherti, Coombes, Katta, Mullis, Wortsman, Schramowski, Kundurthy, Crowson,
  Schmidt, Kaczmarczyk, and Jitsev]{schuhmann2022laion5b}
Christoph Schuhmann, Romain Beaumont, Richard Vencu, Cade Gordon, Ross
  Wightman, Mehdi Cherti, Theo Coombes, Aarush Katta, Clayton Mullis, Mitchell
  Wortsman, Patrick Schramowski, Srivatsa Kundurthy, Katherine Crowson, Ludwig
  Schmidt, Robert Kaczmarczyk, and Jenia Jitsev.
\newblock Laion-5b: An open large-scale dataset for training next generation
  image-text models.
\newblock In \emph{Advances in Neural Information Processing Systems
  (NeurIPS)}, 2022.

\bibitem[Sener and Savarese(2018)]{sener2018active}
Ozan Sener and Silvio Savarese.
\newblock Active learning for convolutional neural networks: A core-set
  approach.
\newblock In \emph{International Conference on Learning Representations
  (ICLR)}, 2018.

\bibitem[Shazeer and Stern(2018)]{shazeer2018adafactor}
Noam Shazeer and Mitchell Stern.
\newblock 2018.

\bibitem[Sinha et~al.(2018)Sinha, Namkoong, and Duchi]{sinha2018certifiable}
Aman Sinha, Hongseok Namkoong, and John Duchi.
\newblock Certifiable distributional robustness with principled adversarial
  training.
\newblock In \emph{International Conference on Learning Representations
  (ICLR)}, 2018.

\bibitem[Snoek et~al.(2012)Snoek, Larochelle, and Adams]{snoek12hyper}
Jasper Snoek, Hugo Larochelle, and Ryan~P. Adams.
\newblock Practical {B}ayesian optimization of machine learning algorithms.
\newblock In \emph{Advances in Neural Information Processing Systems
  (NeurIPS)}, 2012.

\bibitem[Sorscher et~al.(2022)Sorscher, Geirhos, Shekhar, Ganguli, and
  Morcos]{sorscher2022beyond}
Ben Sorscher, Robert Geirhos, Shashank Shekhar, Surya Ganguli, and Ari~S.
  Morcos.
\newblock Beyond neural scaling laws: beating power law scaling via data
  pruning.
\newblock \emph{arXiv}, 2022.

\bibitem[Strubell et~al.(2019)Strubell, Ganesh, and
  McCallum]{strubell2019energy}
Emma Strubell, Ananya Ganesh, and Andrew McCallum.
\newblock Energy and policy considerations for deep learning in {NLP}.
\newblock In \emph{Proceedings of the 57th Annual Meeting of the Association
  for Computational Linguistics}, pages 3645--3650, Florence, Italy, July 2019.
  Association for Computational Linguistics.
\newblock \doi{10.18653/v1/P19-1355}.
\newblock URL \url{https://aclanthology.org/P19-1355}.

\bibitem[Vaswani et~al.(2017)Vaswani, Shazeer, Parmar, Uszkoreit, Jones, Gomez,
  Kaiser, and Polosukhin]{vaswani2017attention}
Ashish Vaswani, Noam Shazeer, Niki Parmar, Jakob Uszkoreit, Llion Jones,
  Aidan~N Gomez, Lukasz Kaiser, and Illia Polosukhin.
\newblock Attention is all you need.
\newblock \emph{arXiv preprint arXiv:1706.03762}, 2017.

\bibitem[Wang et~al.(2019)Wang, Singh, Michael, Hill, Levy, and
  Bowman]{wang2019glue}
Alex Wang, Amapreet Singh, Julian Michael, Felix Hill, Omer Levy, and Samuel~R
  Bowman.
\newblock {GLUE}: A multi-task benchmark and analysis platform for natural
  language understanding.
\newblock In \emph{International Conference on Learning Representations
  (ICLR)}, 2019.

\bibitem[Wang et~al.(2020)Wang, Pham, Michel, Anastasopoulos, Carbonell, and
  Neubig]{wang2020optimizing}
Xinyi Wang, Hieu Pham, Paul Michel, Antonios Anastasopoulos, Jaime Carbonell,
  and Graham Neubig.
\newblock Optimizing data usage via differentiable rewards.
\newblock In \emph{International Conference on Machine Learning (ICML)}, 2020.

\bibitem[Wei et~al.(2015)Wei, Iyer, and Bilmes]{wei2015submodular}
Kai Wei, Rishabh Iyer, and Jeff Bilmes.
\newblock Submodularity in data subset selection and active learning.
\newblock In \emph{International Conference on Machine Learning (ICML)}, 2015.

\bibitem[Xie et~al.(2023)Xie, Santurkar, Ma, and Liang]{xie2023data}
Sang~Michael Xie, Shibani Santurkar, Tengyu Ma, and Percy Liang.
\newblock Data selection for language models via importance resampling.
\newblock \emph{arXiv preprint arXiv:2302.03169}, 2023.

\bibitem[Zoph and Le(2016)]{zoph2016neural}
Barret Zoph and Quoc~V Le.
\newblock Neural architecture search with reinforcement learning.
\newblock \emph{arXiv preprint arXiv:1611.01578}, 2016.

\end{thebibliography}

\appendix
\onecolumn
\begin{figure}
\centering
\begin{subfigure}{0.49\textwidth}
\centering
\includegraphics[width=\textwidth]{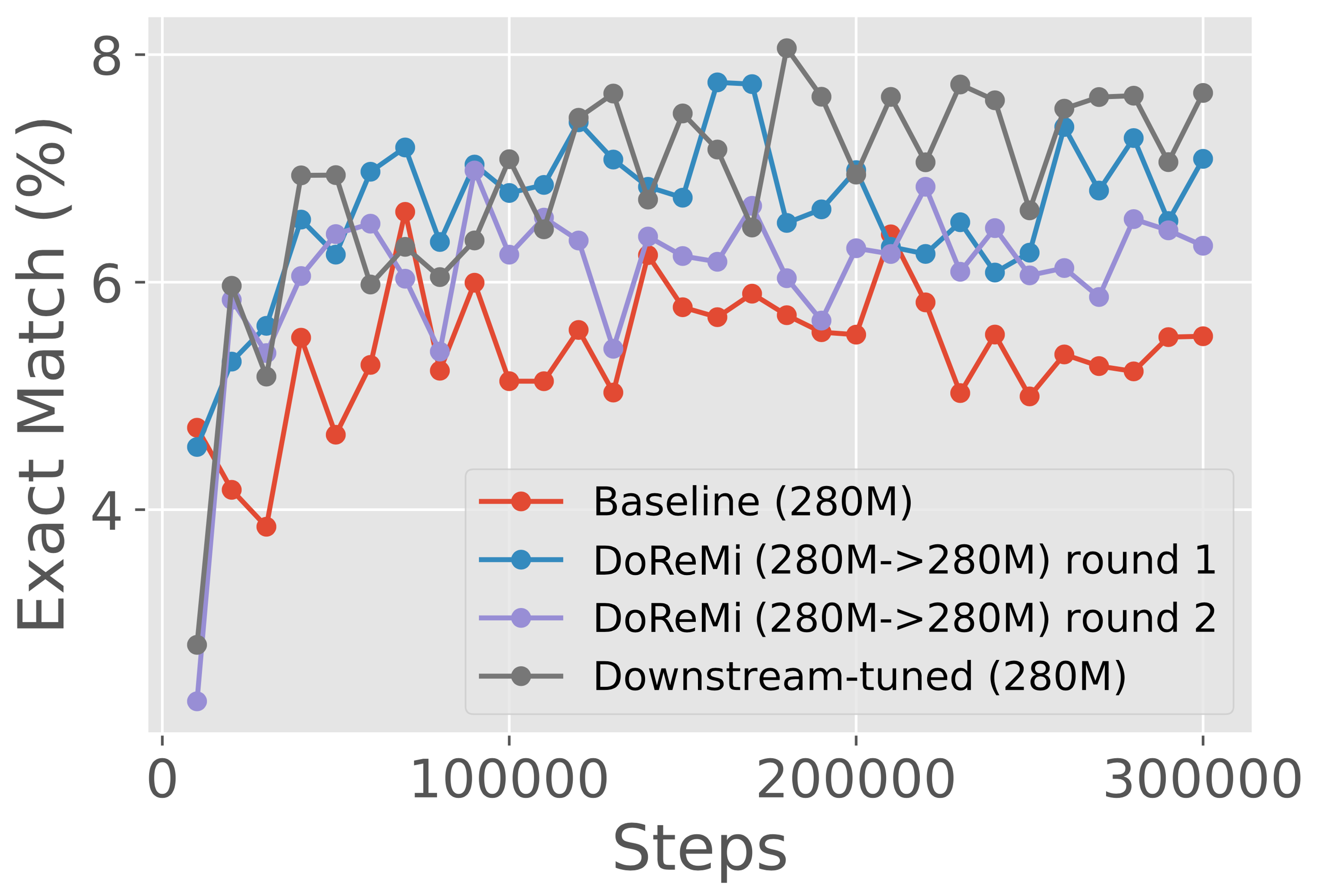}
\caption{280M}
\end{subfigure}
\hfill
\begin{subfigure}{0.49\textwidth}
\centering
\includegraphics[width=\textwidth]{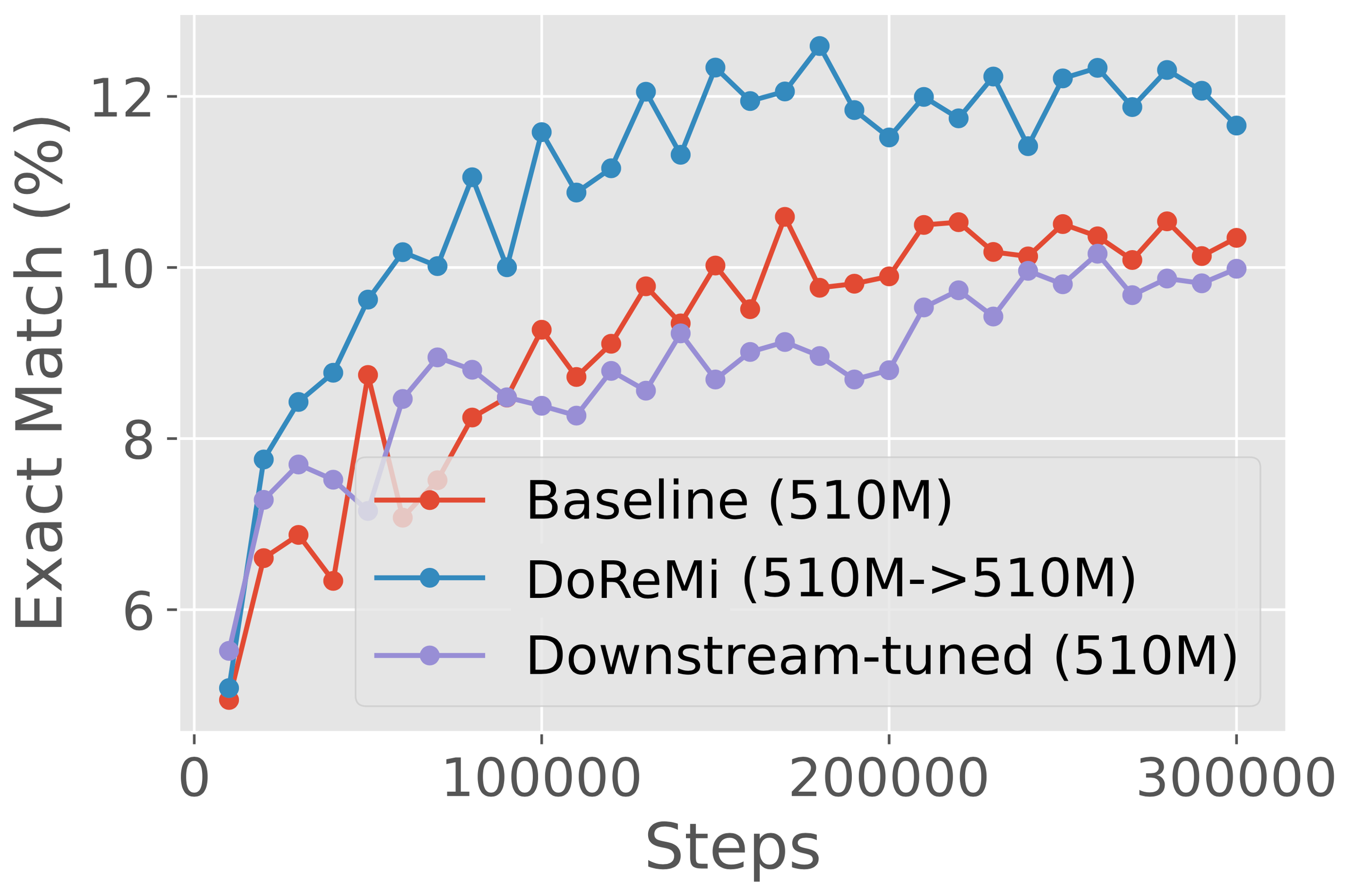}
\caption{510M}
\end{subfigure}
\hfill
\begin{subfigure}{0.49\textwidth}
\centering
\includegraphics[width=\textwidth]{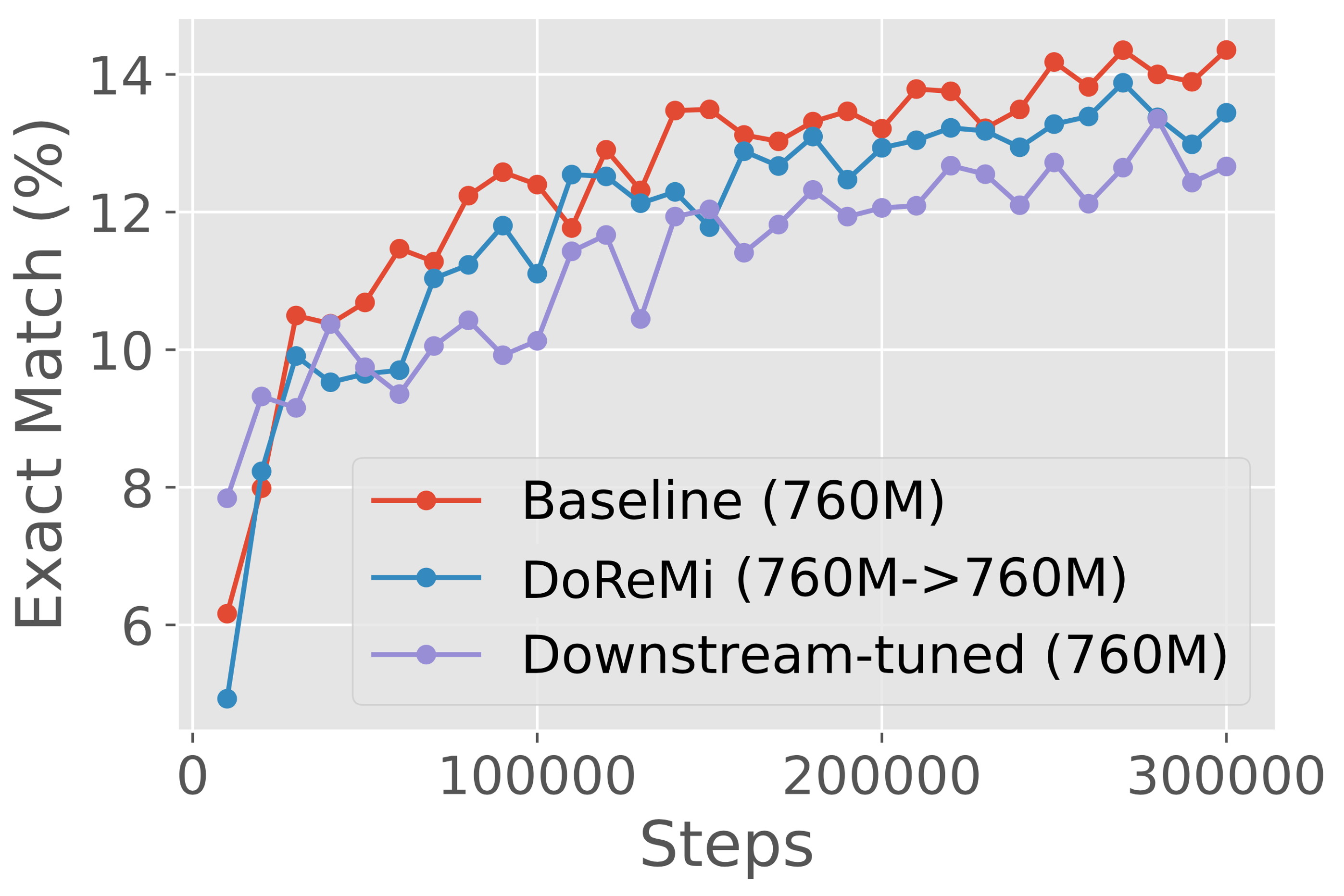}
\caption{760M}
\end{subfigure}
\hfill
\begin{subfigure}{0.49\textwidth}
\centering
\includegraphics[width=\textwidth]{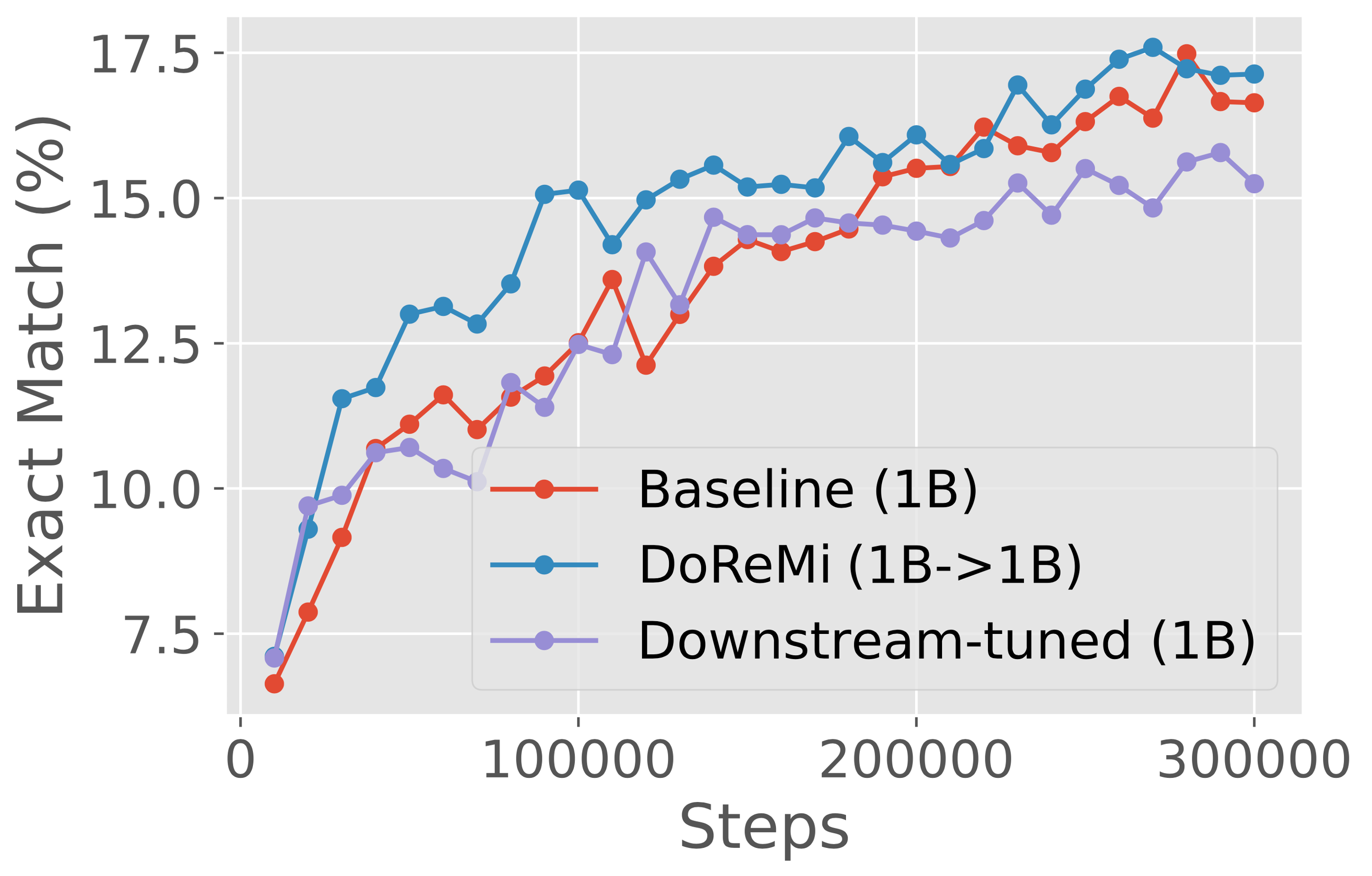}
\caption{1B}
\end{subfigure}
\caption{Average one-shot downstream accuracy across 4 model scales, where the reference/proxy models for \algname are the same size as the final model trained with \algname \weights. All models in this figure are trained on the GLaM dataset. \algname consistently improves downstream accuracy across scales. }
\label{fig:glam-scaling}
\end{figure}

\section{Results Across Scales on the GLaM dataset}
\label{app:glam-scaling}
Figure~\ref{fig:glam-scaling} presents results across different scales (280M, 510M, 760M, 1B) on the GLaM dataset, where the proxy/reference models are the same size as the main model trained with \algname \weights.
Across all scales, \algname is comparable or better than both the baseline (uniform) \weights and downstream-tuned \weights.
Interestingly, for iterated \algname at the 280M scale, the second round weights achieve slightly worse downstream accuracy than the round 1 weights when used to train 280M models, but transfer better to training 8B models.

\begin{table}
\caption{Per-domain log-perplexities for 8B models trained on The Pile where the reference/proxy models are or smaller sizes (70M, 150M, 280M, 1B). Models trained with \algname \weights have lower perplexity on all domains than the baseline weights.}
\label{tab:per-domain-perplexity-pile}
\centering
\begin{adjustbox}{max width=0.9\textwidth}
\begin{tabular}{lrrrrr}
\toprule
                 & Baseline (8B) & \algname (70M->8B) & \algname (150M->8B) & \algname (280M->8B) & \algname (1B->8B) \\
             \midrule
Pile-CC           & 1.64     & 1.51           & 1.48            & 1.41            & 1.55          \\
PubMed Central    & 1.60     & 1.58           & 1.54            & 1.46            & 1.56          \\
Books3            & 1.65     & 1.52           & 1.50            & 1.42            & 1.57          \\
OpenWebText2      & 1.66     & 1.48           & 1.54            & 1.36            & 1.58          \\
ArXiv             & 1.64     & 1.56           & 1.53            & 1.38            & 1.51          \\
Github            & 1.65     & 1.55           & 1.54            & 1.42            & 1.53          \\
FreeLaw           & 1.64     & 1.55           & 1.54            & 1.45            & 1.55          \\
StackExchange     & 1.61     & 1.52           & 1.54            & 1.39            & 1.55          \\
USPTO Backgrounds & 1.70     & 1.53           & 1.50            & 1.41            & 1.56          \\
PubMed Abstracts  & 1.61     & 1.56           & 1.51            & 1.44            & 1.55          \\
Gutenberg (PG-19) & 1.70     & 1.56           & 1.54            & 1.35            & 1.52          \\
OpenSubtitles     & 1.58     & 1.56           & 1.52            & 1.40            & 1.55          \\
Wikipedia (en)    & 1.66     & 1.49           & 1.53            & 1.35            & 1.56          \\
DM Mathematics    & 1.63     & 1.50           & 1.56            & 1.38            & 1.48          \\
Ubuntu IRC        & 1.71     & 1.53           & 1.49            & 1.42            & 1.48          \\
BookCorpus2       & 1.64     & 1.57           & 1.54            & 1.43            & 1.57          \\
EuroParl          & 1.59     & 1.52           & 1.51            & 1.37            & 1.53          \\
HackerNews        & 1.66     & 1.50           & 1.55            & 1.45            & 1.55          \\
YoutubeSubtitles  & 1.67     & 1.63           & 1.55            & 1.42            & 1.53          \\
PhilPapers        & 1.67     & 1.55           & 1.49            & 1.39            & 1.53          \\
NIH ExPorter      & 1.63     & 1.51           & 1.48            & 1.36            & 1.52          \\
Enron Emails      & 1.62     & 1.48           & 1.52            & 1.44            & 1.56         \\
\bottomrule
\end{tabular}
\end{adjustbox}
\end{table}

\begin{table}
\caption{Per-task exact-match accuracies for generative one-shot tasks. All \algname models improve downstream performance significantly over the baseline \weights.}
\label{tab:pile-8b-pertask-generative}
\centering
\begin{adjustbox}{max width=0.9\textwidth}
\begin{tabular}{lrrrrr}
\toprule
                 & Baseline & \algname (1B->8B) & \algname (280M->8B) & \algname (150M->8B) & \algname (70M->8B) \\
                 \midrule
LAMBADA          & 20.10    & 22.55             & 29.19               & 20.59               & 26.20              \\
NaturalQuestions & 4.35     & 6.01              & 7.73                & 6.26                & 5.10               \\
SQuADv2          & 44.43    & 42.22             & 51.89               & 46.53               & 40.99              \\
TriviaQA         & 24.55    & 32.25             & 34.86               & 30.01               & 26.30              \\
WebQuestions     & 6.74     & 8.71              & 9.15                & 9.15                & 6.99               \\
\midrule
Average          & 20.03    & 22.35             & 26.56               & 22.51               & 21.11             \\
\bottomrule
\end{tabular}
\end{adjustbox}
\end{table}

\begin{table}
\caption{Summary of per-domain log-perplexities for 280M, 510M, 760M, and 1B models trained on The Pile, where the reference/proxy models are the same size. \algname improves the worst-case and average perplexity of the baseline \weights in all cases. On average, \algname improves perplexity on 18 out of 22 domains.  }
\label{tab:scaling-perplexity-pile}
\centering
\begin{adjustbox}{max width=0.9\textwidth}
\begin{tabular}{lrrrrr}
\toprule
                    & Worst-case log-ppl & Avg log-ppl & \# domains beating baseline \\
                    \midrule
Baseline (280M)     & 2.39            & 2.32     & 0/22                          \\
\algname (280M->280M)   & 2.19            & 2.13     & 22/22                         \\
Proxy (280M) & 2.33            & 2.27     & 19/22                         \\
\midrule
Baseline (510M)     & 2.14            & 2.08     & 0/22                          \\
\algname (510M->510M)   & 2.14            & 2.06     & 15/22                         \\
Proxy (510M) & 2.23            & 2.18     & 0/22                          \\
\midrule
Baseline (760M)     & 2.05            & 1.97     & 0/22                          \\
\algname (760M->760M)   & 2.00               & 1.94     & 17/22                         \\
Proxy (760M) & 2.15            & 2.10      & 0/22                          \\
\midrule
Baseline (1B)       & 1.94            & 1.87     & 0/22                          \\
\algname (1B->1B)       & 1.92            & 1.83     & 19/22                         \\
Proxy (1B)   & 2.11            & 2.02     & 0/22               \\
\bottomrule
\end{tabular}
\end{adjustbox}
\end{table}

\begin{table}
\caption{Summary of perplexity results for ablations on the DRO objective (excess loss). The individual components (which prefer hardest and easiest domains respectively) do not reduce perplexity over the baseline. }
\label{tab:ablations-perplexity-pile}
\centering
\begin{adjustbox}{max width=0.9\textwidth}
\begin{tabular}{lrrrrr}
\toprule
                                                  & Worst-case log-ppl & Avg log-ppl & \# domains beating baseline \\
                                                  \midrule
Baseline (280M)                      & 2.39            & 2.32     & 0                          \\
\algname (280M->280M)           & 2.19            & 2.13     & 22/22                         \\
Hardest (280M->280M)          & 2.66            & 2.62     & 0/22                          \\
Easiest (280M->280M)          & 4.27            & 4.18     & 0/22                          \\
\bottomrule
\end{tabular}
\end{adjustbox}
\end{table}

\section{Detailed Results for The Pile}
\label{app:full-results-pile}

\paragraph{Per-domain perplexities for 8B models.}
Table~\ref{tab:per-domain-perplexity-pile} shows per-domain perplexities for 8B models trained on the Pile. The reference/proxy models in this case are 70M, 150M, 280M, and 1B. \algname improves the perplexity on each domain compared to the baseline \weights.

\paragraph{Per-task accuracies for 8B models.}
Table~\ref{tab:pile-8b-pertask-generative} shows the accuracies on one-shot generative tasks for various reference/proxy model sizes from 70M to 1B. All \algname models improve downstream performance significantly over the baseline.

\paragraph{Summary of perplexity results across scales.}
Table~\ref{tab:scaling-perplexity-pile} shows a summary of per-domain perplexities for \algname across 4 scales (280M, 510M, 760M, 1B). Here, the reference/proxy models are the same size as the main model trained with \algname \weights.
On average, \algname improves perplexity on 18.25 out of 22 domains from The Pile. The worst-case perplexity is always reduced (or comparable in the 510M case) with respect to the baseline \weights.

\paragraph{Perplexity results for ablations.}
Table~\ref{tab:ablations-perplexity-pile} shows the perplexities for ablations on the DRO objective. We change the DRO objective and use these to tune \weights on 280M reference/proxy models. These tuned \weights are then used to train a main 280M model. Hardest refers to optimizing the domain-level log-perplexity without baselining with a reference model. Easiest refers to optimizing for the domains with lowest log-perplexity under the reference model.
Both ablations do not improve perplexity on any domain over the baseline.
Optimizing for the ``hardest'' domain does not actually result in improving worst-case perplexity, supporting the results of~\citet{oren2019drolm}, which also employs DRO for language modeling with a baselined loss.

\begin{figure}
\centering
\begin{subfigure}{1.0\textwidth}
\centering
\includegraphics[width=\textwidth]{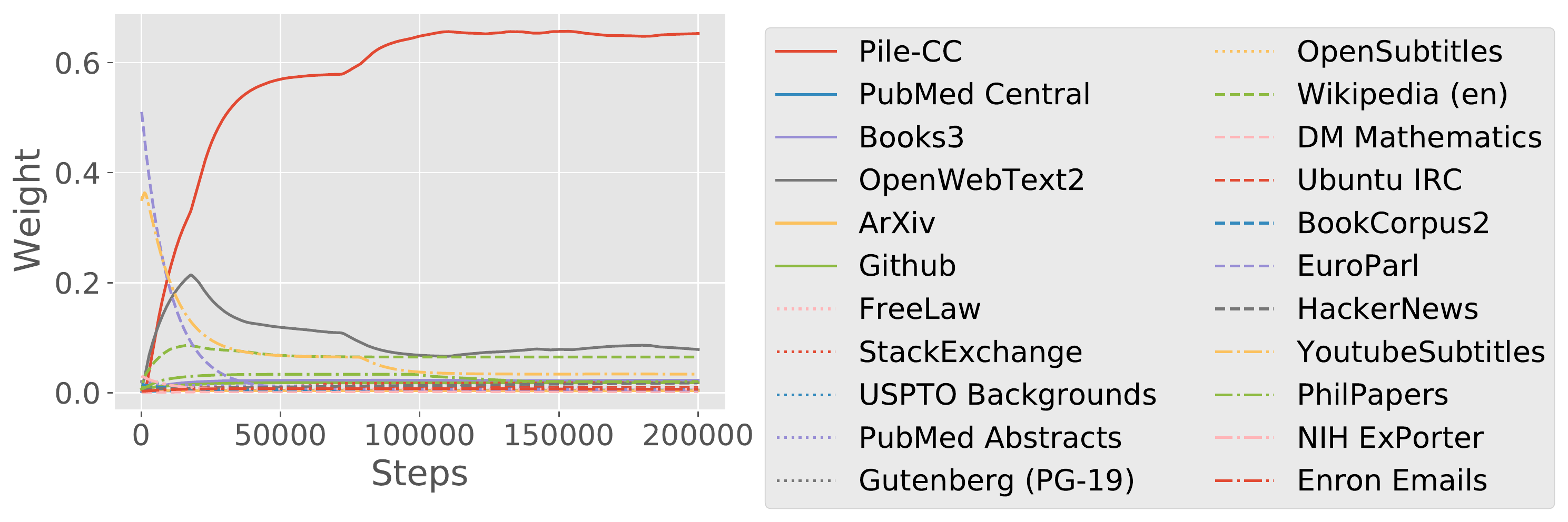}
\caption{280M}
\end{subfigure}
\hfill
\begin{subfigure}{1.0\textwidth}
\centering
\includegraphics[width=\textwidth]{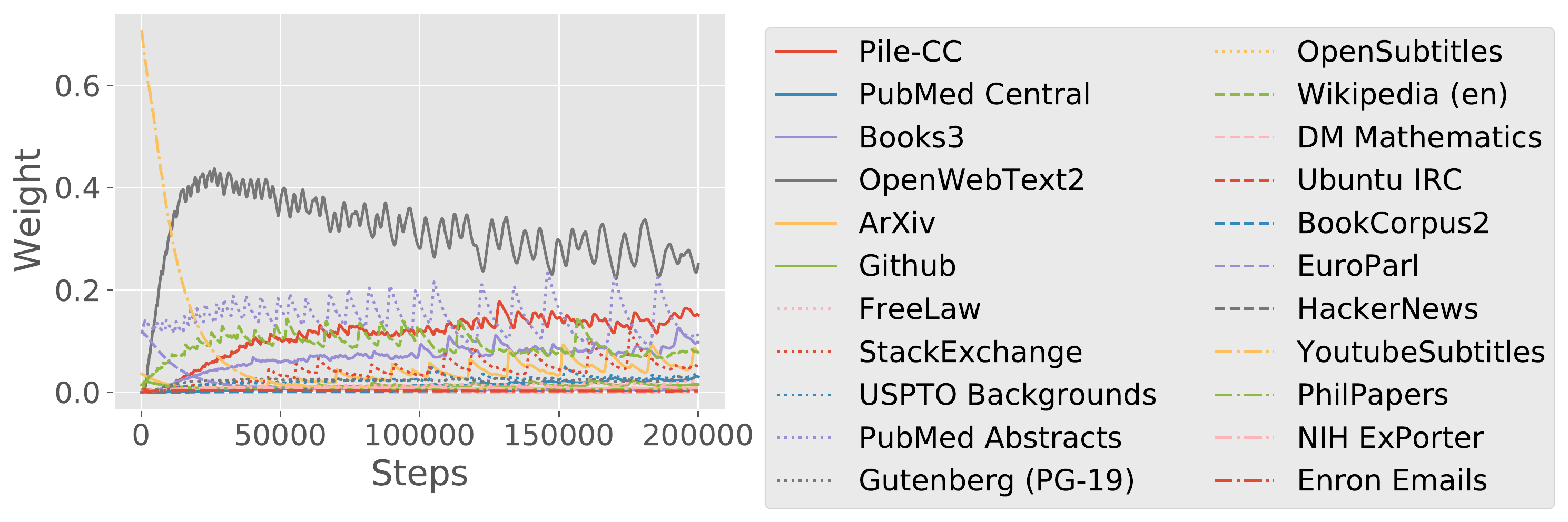}
\caption{1B}
\end{subfigure}
\caption{Exponential moving average of \weights throughout a \algname run for 280M and 1B reference/proxy models. In the beginning of the run, the \weights change quickly and then become more stable after 50k steps. This suggests that 1) smaller compute budgets may require drastically different \weights, and 2) we may be able to save compute by extrapolating the \weights after 50k steps.}
\label{fig:pile-weights-evolution}
\end{figure}

\paragraph{Trajectory of \weights.}
Figure~\ref{fig:pile-weights-evolution} shows the exponential moving average (smoothing parameter 0.99) of \weights during a run of \algname. In both cases, there are domains with very high weight initially and decrease in weight very quickly (within 50k steps). Since we compute the final \weights by integrating these curves over steps and normalizing, this suggests that if we have a smaller compute budget, these domains could become more important --- this highlights the dependence of the mixture weights on the compute budget. At the same time, the \weights tend to quickly stabilize after 50k steps, suggesting that the optimal \weights should be similar for larger compute budgets.
We may also be able to take advantage of this stability after 50k steps to run \algname for a smaller number of steps and extrapolate the \weights to save compute.

\begin{table}
\caption{\Weights on The Pile. Baseline \weights are computed from the default Pile dataset. With different proxy model sizes, \algname (280M) and \algname (1B) result in different \weights. Despite the differences, the qualitative patterns are similar other than the which web domain has the most weight.
}
\label{tab:pile-mixture-weights-1B}
\centering
\begin{adjustbox}{max width=0.48\textwidth}
\begin{tabular}{lrrrr}
\toprule
                  & Baseline & \algname (280M) & \algname (1B) \\
                  \midrule
Pile-CC           & 0.1121   & 0.6057      & 0.1199    \\
PubMed Central    & 0.1071   & 0.0046      & 0.0149    \\
Books3            & 0.0676   & 0.0224      & 0.0739    \\
OpenWebText2      & 0.1247   & 0.1019      & 0.3289    \\
ArXiv             & 0.1052   & 0.0036      & 0.0384    \\
Github            & 0.0427   & 0.0179      & 0.0129    \\
FreeLaw           & 0.0386   & 0.0043      & 0.0148    \\
StackExchange     & 0.0929   & 0.0153      & 0.0452    \\
USPTO Backgrounds & 0.0420   & 0.0036      & 0.0260    \\
PubMed Abstracts  & 0.0845   & 0.0113      & 0.1461    \\
Gutenberg (PG-19) & 0.0199   & 0.0072      & 0.0250    \\
OpenSubtitles     & 0.0124   & 0.0047      & 0.0017    \\
Wikipedia (en)    & 0.0919   & 0.0699      & 0.0962    \\
DM Mathematics    & 0.0198   & 0.0018      & 0.0004    \\
Ubuntu IRC        & 0.0074   & 0.0093      & 0.0044    \\
BookCorpus2       & 0.0044   & 0.0061      & 0.0029    \\
EuroParl          & 0.0043   & 0.0062      & 0.0078    \\
HackerNews        & 0.0075   & 0.0134      & 0.0058    \\
YoutubeSubtitles  & 0.0042   & 0.0502      & 0.0159    \\
PhilPapers        & 0.0027   & 0.0274      & 0.0063    \\
NIH ExPorter      & 0.0052   & 0.0063      & 0.0094    \\
Enron Emails      & 0.0030   & 0.0070      & 0.0033   \\
\bottomrule
\end{tabular}
\end{adjustbox}
\end{table}

\paragraph{Comparison of \weights for 280M and 1B.}
Table~\ref{tab:pile-mixture-weights-1B} presents the \algname \weights for The Pile at 280M and 1B proxy models.
Different proxy model sizes can result in different \weights, which suggests that there may be multiple local minima in \weight space.
With a 280M proxy model, most of the weight is put on the Pile-CC web text domain, while \algname with a 1B proxy model puts most of the weight on OpenWebText2.
The overall pattern of the \weights for the rest of the domains are similar.

\section{Training Details}
\label{app:training-details}

\paragraph{Data preprocessing.}
For all datasets, we preprocessed the data by chunking into length 1024 examples with respect to a SentencePiece tokenizer with 256k vocabulary size. The examples are separated by domain to facilitate hierarchical sampling (first sample a domain according to some \weights, then sample an example from that domain at random).
To reduce the amount of padding tokens, we made an effort to pack examples (possibly from different domains) together into the same sequence.
When doing such a packing, we compute the domain perplexities on a per-token level in \algname.

\paragraph{Baseline \weights for The Pile.}
The baseline \weights for The Pile were computed from The Pile dataset and the number of epochs for each domain given in~\citet{gao2020pile}. After chunking into length 1024 examples, we counted the number of examples in each domain and multiplied by the number of epochs that domain specified in~\citet{gao2020pile}.
We then normalized these counts to obtain the baseline \weights.

\paragraph{Training setup.}
For all training runs (including DRO runs), we train with a batch size of 512, initial learning rate of 1e-3, weight decay of 1e-2, and gradient clipping to norm 1. We decay the learning rate exponentially until it reaches a minimum of 1e-4 at the end of training, with a linear warmup of 6\% of the total training steps.
We train for 200k steps on The Pile and 300k steps on the GLaM dataset.
Models under 1B parameters were trained with TPUv3 accelerators, while 1B and 8B models were trained with TPUv4.

\paragraph{Model architectures.}
Table~\ref{tab:archictectures} shows the architecture hyperparameters for the model sizes used in the paper. All the models we use are vanilla Transformer decoder-only models with a 256k vocab size. 

\begin{table}
\caption{Architecture hyperparameters for various model scales used in the paper. All models are vanilla Transformer decoder-only models and use vocabulary size 256k.}
\label{tab:archictectures}
\centering
\begin{adjustbox}{max width=0.9\textwidth}
\begin{tabular}{lrrrrr}
\toprule
     & Layers & Attention heads & Attention head dim & Model dim & Hidden dim \\
     \midrule
70M  & 3      & 4               & 64                 & 256       & 1024       \\
150M & 6      & 8               & 64                 & 512       & 2048       \\
280M & 12     & 12              & 64                 & 768       & 3072       \\
510M & 12     & 16              & 64                 & 1024      & 8192       \\
760M & 12     & 20              & 64                 & 1280      & 8192       \\
1B   & 16     & 32              & 64                 & 2048      & 8192       \\
8B   & 32     & 32              & 128                & 4096      & 24576   \\
\bottomrule
\end{tabular}
\end{adjustbox}
\end{table}

\section{Simple Example Where Data Reweighting Has No Tradeoff}
\label{sec:simple-example}
Motivated by the findings in Section~\ref{sec:main-results-8B}, we present a simple language modeling example where reweighting the training data from different domains improves perplexity on all domains.  
The example shows that \algname downweights domains that are extremely high or low entropy.

\paragraph{Setup.}
Suppose the ground-truth distribution of text $\pstar$ is a mixture over $k$ domains, where each domain $z \in \{1,\dots, k\}$ is defined by a different unigram distribution $\pstar(x \mid z)$ over $m$ tokens.
Given a budget of $n$ training samples, the goal is choose \weights $p(z)$ ($k$ scalars that add to 1) to sample training data with such that we learn the parameters of the unigram distributions $\pstar(\cdot \mid z)$ well for all $z$ from $1$ to $k$. Notably, we do not aim to estimate the ground truth mixture proportions across domains.

\paragraph{Data.}
Given some \weights $p(z)$, we sample training data hierarchically: first we determine the number of samples $n_z$ per domain $z$ by drawing from a multinomial distribution over $k$ possibilities with probabilities defined by $p(z)$ and $n$ total trials.
Then, for each domain $z$, we sample $n_z$ tokens from $\pstar(\cdot \mid z)$, forming a vector of tokens $X_z$ with length $n_z$.

\paragraph{Model.} 
For each domain $z$, we consider a Bayesian model of the unigram distribution $p(x \mid z;\theta)$ with a Dirichlet prior $p(\theta \mid z; \beta)$ over the unigram distribution parameters $\theta \in \Delta^m$.
The Dirichlet prior has hyperparameters $\beta \in \R^m$, which can be viewed as a ``pseudo-count'' for each token.
For each domain $z$, we estimate the parameters $\hat{\theta}_z$ by computing the mean of the posterior distribution conditioned on the data:
\begin{align}
\label{eqn:simple-estimator}
\thetaxprimex &= \frac{1}{n_z + s_z} \left[ \lambdaxprimex + \sum_{i=1}^{n_z} \indicator[\examplexi=x] \right] \text{~~for all $x \in \{1,\dots,m\}$}
\end{align}
where $s_z = \sum_{x} \lambdaxprimex$ is the sum of pseudocounts.

For a domain $z$, we can write the parameter error of this estimator as a function of the ``difficulty'' $H_z$ of predicting the next token and the ``quality'' of the prior $\Delta_z$, defined below.
\begin{lemma}
\label{lem:param-error-simple}
For domain index $z$ with $n_z$ samples, the parameter error is
\begin{align}
\sum_{x} \E[(\thetaxprimex - \pstar(x\mid z))^2] &=
\frac{n_z H_z + s_z^2 \Delta_z}{(n_z + s_z)^2}
\end{align}
where
\begin{align}
    H_z &= \sum_{x}\pstar(x \mid z)(1-\pstar(x\mid z))\\
   \Delta_z &= \sum_{x} \left(\pstar(x \mid z) - \frac{\lambdaxprimex}{s_z}\right)^2.
\end{align}
\end{lemma}
\begin{proof}
The parameter error is
\begin{align}
\sum_{x} \E[(\thetaxprimex - \pstar(x\mid z))^2] &= \sum_{x} \E[\thetaxprimex^2] - 2\E[\thetaxprimex]\pstar(x\mid z) + \pstar(x \mid z)^2.
\end{align}
Evaluating the terms separately,
\begin{align}
\E[\thetaxprimex] &= \frac{1}{n_z + s_z} \left[ \lambdaxprimex + \sum_{i=1}^{n_z} \indicator[\examplexi=x] \right]\\
&= \frac{1}{n_z + s_z}(\lambdaxprimex + n_z \pstar(x\mid z))\\
\E[\thetaxprimex^2] &= \frac{1}{(n_z+s_z)^2}\E[(\lambdaxprimex + \sum_{i=1}^{n_z} \indicator[\examplexi = x])^2]\\
&= \frac{1}{(n_z+s_z)^2}\left[\lambdaxprimex^2 + 2\lambdaxprimex n_z\pstar(x\mid z) + n_z \pstar(x \mid z) + (n_z^2 - n_z)\pstar(x\mid z)^2\right]
\end{align}
Putting it all together, the parameter error can be written as
\begin{align}
\sum_{x} \E[(\thetaxprimex - \pstar(x\mid z))^2] &= \sum_{x} \frac{(s_z^2 - n_z)\pstar(x\mid z)^2 + \lambdaxprimex^2 + (n_z - 2s_z\lambdaxprimex)\pstar(x\mid z)}{(n_z + s_z)^2 }\\
&= \sum_{x} \frac{n_z \pstar(x \mid z)(1-\pstar(x \mid z)) + s_z^2\left(\pstar(x\mid z) - \frac{\lambdaxprimex}{s_z} \right)^2}{(n_z+s_z)^2}\\
&=\frac{n_z H_z + s_z^2 \Delta_z}{(n_z + s_z)^2}.
\end{align}
\end{proof}

\paragraph{No-tradeoff example.}
Suppose there are 3 domains $z \in \{1,2,3\}$ and $m=3$ vocabulary tokens $x \in \{1,2,3\}$.
We use a symmetric Dirichlet prior (preferring a uniform token distribution) where $\lambdaxprimex = 1/3$ for all tokens $x$ and domains $z$.
Here, $s_z = \sum_{x} \lambdaxprimex = 1$.
In this setting, we show that there is a set of \weights that has strictly lower parameter error than the baseline where we sample the same number of tokens from each domain: $n_z$ are equal for all domains $z$.

Suppose the ground truth paramaters for the unigram distributions are
\begin{align}
\label{eqn:simple-transmat}
    \begin{bmatrix}
        1 & 0 & 0\\
        0.7 & 0.2 & 0.1 \\
        1/3 & 1/3 & 1/3
    \end{bmatrix},
\end{align}
where row $z$ contains the parameters for domain $z$. For example, token 1 has probability 1 under domain 1's unigram distribution. 

For domain $z=1$ (non-noisy domain), we have $H_1=0$ so the parameter error (according to Lemma~\ref{lem:param-error-simple}) is
\begin{align}
    \frac{s_1^2 \Delta_1}{(n_1 + s_1)^2}
\end{align}
which is strictly decreasing in the number of samples $n_1$.

For domain $z=3$ (noisy domain), we have $\Delta_3=0$ so the parameter error is
\begin{align}
\frac{n_3 H_3}{(n_3 + s_3)^2},
\end{align}
by Lemma~\ref{lem:param-error-simple}.
This error is minimized to zero at $n_3=0$ (no samples).
This means that we can allocate samples elsewhere while still reducing error.

For $z=2$ (intermediate entropy domain), we have $\Delta_2 = 0.207$ and $H_2 = 0.46$.
The derivative of the parameter error with respect to the number of samples $n_2$ is
\begin{align}
    \frac{\partial}{\partial n_2} \frac{n_2 H_2 + s_2^2 \Delta_2}{(n_2 + s_2)^2} = \frac{H_2(s_2 - n_2) - 2 s_2^2 \Delta_2}{(n_2 + s_2)^3}
\end{align}
which is negative when
\begin{align}
    n_2 > s_2 - \frac{2s_2^2\Delta_2}{H_2}.
\end{align}
This inequality holds in this case since $\frac{2\Delta_2}{H_2} < 1$ and $s_2 = 1$. Therefore the parameter error is decreasing in the number of samples $n_2$.

Thus, any \weights that reallocate the examples from domain 3 to domains 1 and 2 reduces the parameter error for all domains.

\paragraph{What kind of domains are downweighted?}
Intuitively, we can downweight the very noisy (high entropy/difficulty) domain 3 because the initialization perfectly matches the ground truth. This allows us to reallocate samples to the other domains 1 and 2. Between these, domain 1 requires less additional samples since the parameter error decreases very quickly with the number of samples $n_1$ (the difficulty $H_1$ is zero). Thus, the easiest domains should also receive relatively less weight. In practice, positive transfer between domains (which is not captured here) can also contribute to scenarios where reweighting results in no tradeoff across domains.

\paragraph{Simulation with \algname.}
We consider running \algname on the above no-tradeoff instance of the simple example with the ground truth unigram distributions in Equation~\ref{eqn:simple-transmat}.
Note that \algname's domain reweighting step (Step 2, Algorithm~\ref{alg:alg1}) involves a loop over $T$ iterative model updates, while the estimator from Equation~\ref{eqn:simple-estimator} is computed in closed form.
To adapt the estimator for \algname, we consider an iterative version where the average is computed in an online fashion.
We run \algname for $T=500$ steps using minibatch size 1 over the $n=500$ training examples with \weight update rate $\eta=0.5$. For the model update at step $t$ on an example $x$ from domain $z$, we increase the pseudo-count $\thetaxprimex$ by the current \weight $\alpha_t$ corresponding to domain $z$.
Instead of using the examples in the minibatch (which is only size 1 and doesn't represent all domains), we compute the per-domain excess log-perplexities in Algorithm~\ref{alg:alg1} using a fixed, independent evaluation set of 30 examples.

We compare \algname against a model trained with baseline \weights, which are uniform over the 3 domains. All models are trained on $n=500$ training examples.
We evaluate the log-perplexity of a model on each domain in closed form using the ground truth unigram distribution parameters.

On this simple example, \algname returns \weights $[0.39, 0.61, 0.0]$ after rounding to 2 decimal places. These weights correspond to our intuitions --- the first domain (non-noisy) is increased by a small amount, the third domain (noisy) is decreased to 0 weight, and most of the weight is allocated to the second domain. We use these \weights to generate a new dataset of 500 examples. The model trained with this new dataset improves over the baseline model in perplexity on all domains.

\end{document}